\documentclass[journal]{IEEEtran}
\ifCLASSINFOpdf
\else
\fi

\usepackage{amssymb}
\usepackage[colorlinks=False,linkcolor=blue]{hyperref}
\usepackage{amsmath}
\usepackage{graphicx}
\usepackage{caption}
\usepackage[dvipsnames]{xcolor}
\usepackage{subcaption}
\graphicspath{ {images/} }
\usepackage{float}
\usepackage{amsthm}
\usepackage{listings}
\usepackage{algorithm,algcompatible}
\makeatletter
\def\munderbar#1{\underline{\sbox\tw@{$#1$}\dp\tw@\z@\box\tw@}}
\makeatother
\newcommand{\norm}[1]{\left\lVert#1\right\rVert}
\usepackage{algorithm}
\usepackage{algorithmicx}
\usepackage{algpseudocode}
\usepackage{makecell}
\usepackage{booktabs}       
\usepackage{url}
\usepackage{multirow}
\usepackage{lipsum}
\usepackage{bbm}
\usepackage{soul}
\usepackage{multirow}

\newtheorem{theorem}{Theorem}

\newtheorem{problem}{Problem}
\newtheorem{remark}{Remark}

\newtheorem{definition}{Definition}

\newcommand{\R}{\mathbb{R}}
\newcommand{\Z}{\mathbb{Z}}

\newcommand{\s}{\mathbf{s}}

\begin{document}
%
\title{A Finite-Sampling, Operational Domain Specific, and Provably Unbiased Connected and Automated Vehicle Safety Metric}
%
%
%

\author{Bowen Weng, Linda Capito,  Umit Ozguner, \textit{IEEE Life Fellow}, Keith Redmill, \textit{IEEE Senior Member}
\thanks{Bowen Weng, Linda Capito, Umit Ozguner, and Keith Redmill are with Department of Electrical and Computer Engineering at Ohio State University, OH, USA.}
}

\maketitle

\begin{abstract}
A connected and automated vehicle safety metric determines the performance of a subject vehicle (SV) by analyzing the data involving the interactions among the SV and other dynamic road users and environmental features. When the data set contains only a finite set of samples collected from the naturalistic mixed multi-modal traffic driving environment, a metric is expected to generalize the safety assessment outcome from the observed finite samples to the unobserved cases by specifying in what domain the SV is expected to be safe and how safe the SV is, statistically, in that domain. However, to the best of our knowledge, none of the existing safety metrics is able to justify the above properties with an operational domain specific, guaranteed complete, and provably unbiased safety evaluation outcome. In this paper, we propose a novel safety metric that involves the $\alpha$-shape and the $\epsilon$-almost robustly forward invariant set to characterize the SV's almost safe operable domain and the probability for the SV to remain inside the safe domain indefinitely, respectively. The empirical performance of the proposed method is demonstrated in several different operational design domains through a series of cases covering a variety of fidelity levels (real-world and simulators), driving environments (highway, urban, and intersections), road users (car, truck, and pedestrian), and SV driving behaviors (human driver and self driving algorithms). 
\end{abstract}

\begin{IEEEkeywords}
Safety Metric, Invariant Set, Connected and Automated Vehicle, Operational Design Domain
\end{IEEEkeywords}

%
\IEEEpeerreviewmaketitle

\section{Introduction}\label{sec:introduction}
\IEEEPARstart{V}{ehicles} equipped with Advanced Driver Assist Systems (ADAS) or Automated Driving Systems (ADS), including the full spectrum of Connected and Automated Vehicles (CAVs), operate in a mixed multi-modal traffic environment with various traffic participants (e.g., pedestrians, cyclists, and different types of vehicles) and environmental disturbances (e.g., road gradients, surface friction, and weather conditions). In general, to ensure the safe performance of a Subject Vehicle (SV) or a fleet of SVs (e.g., a group of CAVs) in the real-world mixed traffic driving environment (also referred to as the naturalistic driving environment in the literature~\cite{feng2020testing} and the ``nominal driving environment" in the remainder of this paper), one typically follows a two-step procedure with testing and analysis. First, the \textit{testing} procedure deploys the SV (or a fleet of SVs) in the environment and acquires the traffic interactions and other observable infrastructure information~\cite{altekar2021infrastructure} near each SV with a certain data acquisition system. This creates a set of finite observations sampled from the nominal driving environment. Note that the environment can consist of simulated scenes, real-world on-road operation, or controlled testing and proving grounds. Second, the \textit{analysis} procedure summarizes the safety performance from the finite sampled observations and seeks to generalize the understanding, intuitively or provably, to the non-sampled unobserved cases. In this paper, we only focus on the offline analysis case, which involves presenting previously collected data, as it stands, to one or more specific safety metrics.

Let's start from a toy example of one observing a certain SV operating safely (without collisions, human driver engagements, or breaking traffic rules) navigating from the Empire State Building to Times Square (both are attractions in New York City, United States) for one mile at 6 P.M. on a weekday through a crowd of vehicles, pedestrians, cyclists, and an intersection with traffic lights. A safety measure then seeks to infer the SV's overall safety performance in the mixed-traffic driving environment from the data collected during the one-mile observation. 

The first class of measures are known as leading measures as they ``reflect performance, activity, and prevention"~\cite{fraade2018measuring}, such as infractions (i.e., noncriminal violations of state and local traffic law)~\cite{censi2019liability} and disengagements~\cite{favaro2018autonomous}. In general, it is expected that the leading measure outcomes from the one-mile trip imply a certain safety property, yet such an implication is mostly intuitive (e.g., the observed $0\%$ engagement rate within one mile does not necessarily hold for the rest of a longer trip). 

On the other hand, the lagging measures are primarily interested in safety outcomes or harm~\cite{fraade2018measuring}. They can be further classified as observed failures, predictive failures, and inferred failures. As a collision is the most well-adopted failure event in the literature, it is considered interchangeable with failure for the remainder of this paper.

The observed failures share the same spirit with many aforementioned leading measures. For example, the $0\%$ observed collision rate in the one-mile trip does not necessarily hold as the vehicle operation proceeds. It can also be expanded from the scalar value measure to a more complex group of collision ratings~\cite{schwall2020waymo}, yet the above mentioned problem still remains. Second, the predictive failure is often derived by asserting surrogate models and assumptions~\cite{bowen2020presentation, weng2021model}. Hence some lagging safety measures are also referred to as surrogate safety measures in the literature~\cite{wang2021review}. One well-adopted assumption and the surrogate model is the steady-state assumption (all road users maintain the current velocity and heading) and the linear double integrator dynamics, leading to a series of classic safety measures including time-to-collision (TTC)~\cite{lee1976theory} and the minimum safe distance (MSD) based variants~\cite{wishart2020driving}. Some recently propose metrics, developed as more complex dynamic and behavioral models are considered, include the Responsibility Sensitive Safety (RSS)~\cite{shalev2017formal} based method, Instantaneous Safety Metric (ISM)~\cite{every2017novel}, criticality metric~\cite{junietz2018criticality}, and Model Predictive Instantaneous Safety Metric (MPrISM)~\cite{weng2020model}, which all belong to a class of model predictive safety measures~\cite{weng2021model}. Note that many of the lagging measures generalize the finite observations to the non-sampled cases to some extent, but the generalization relies heavily on asserted surrogate models and assumptions which are mostly invalid in the real-world mixed traffic driving environment~\cite{bowen2020presentation, weng2021model}. 

Finally, in contrast to the predictive failure based lagging measures that generalize the safety assessment to an ``imaginary" domain that does not necessarily align with the nominal driving environment~\cite{bowen2020presentation, weng2021model}, the statistically inferred failure rate estimate is an unbiased safety assessment generalization from the observed samples to the nominal driving environment. One representative method in this category comes from Fraade et al.~\cite{fraade2018measuring} using the Monte-Carlo sampling approach to provide the finite-sampling safety assurance by inferring the SV's fatality rate estimate from consecutively operating for a certain number of miles safely. If applied to the aforementioned one-mile trip example, with confidence level 90\%, the SV has a fatality rate of 90 million fatalities per 100 million miles. Despite the 90\% risk being rigorously provable, note that the safety measure outcome is essentially invariant from the mixed traffic environment as one still obtains the same values if the vehicle safely operates on the same route on empty streets at 3 A.M. (i.e., no other traffic objects are present). Note that the importance sampling based technique~\cite{ding2011toward} has been shown capable of improving the sampling efficiency of the Monte-Carlo sampling methods. However, the accuracy of the estimated failure rate relies heavily on the accuracy of the estimated importance function, which is not a provable condition in general. 

Another line of research on formal safety analysis relies on a model-based approach where one first approximates a certain probabilistic model, parametric or non-parametric, from the observed data and then derives the risk rate estimate~\cite{aasljung2019probabilistic}, information gain~\cite{collin2021plane}, and other safety related properties~\cite{hejase2020methodology} using the obtained model. This shares a similar problem with the aforementioned importance sampling based methods as the safety outcome estimate is unbiased only if the approximated model is also unbiased with analytically justifiable variance, which remains as an open challenge to date.

To a certain extent, existing efforts seek to establish a CAV safety measurement that is monotonic w.r.t. the SV's safety performance (e.g., a lower TTC value indicates a more unsafe SV behavior than a higher TTC value). This is generally true if other variables are controlled properly. One particularly important variable is the SV's operable domain. As we have discussed before, the leading measures fail to control the domain variables since the generalization is biased. The predictive failure based lagging measures also fail, for while the generalization is provably true in a certain predictive domain, it does not necessarily align with the nominal driving environment. Finally, for the statistically inferred collision rate~\cite{fraade2018measuring}, the operable domain is invariant as the required total mileage to claim a certain fatality rate with a given confidence level does not change as one moves from the lead-vehicle following domain to a more complex operable domain involving mixed-traffic interactions. Moreover, the particular SV driving behavior also partially affects its operable domain construction. As a result, the notion of \emph{one vehicle being safer than the other} is mostly problematic as it is essentially a multi-dimensional comparison. This will be demonstrated in detail through a series of examples in Section~\ref{sec:case_study}.

In summary, to make a competitive safety measurement for the SV that resolves the various mentioned problems of existing methods, the following two questions need to be jointly and rigorously addressed:
\begin{itemize}
    \item Q1: \textbf{Where} (in terms of the operable domain) will the vehicle be statistically safe within the nominal driving environment?
    \item Q2: Supplied with a certain operable domain, \textbf{how} safe will the vehicle be within the given domain?
\end{itemize}
To the best of our knowledge, there does not exist a safety metric that rigorously addresses the above two questions simultaneously. 

In this paper, we propose a novel safety metric using the $\alpha$-shape~\cite{akkiraju1995alpha}, a set of piece-wise linear segments that characterize the extent and shape of a finite set of points, and the $\epsilon$-almost invariant set~\cite{weng2021towards,weng2021formal}, an almost forward invariant set except for an arbitrarily small sub-set induced by the probability coefficient $\epsilon$. Given the driving data collected from a certain testing procedure, the proposed method first rearranges the data to formulate the \emph{Operational State Space} (OSS) of a multi-agent system that admits measurable states and other non-observable uncertainties. One then characterizes an \emph{Operational Design Domain} (ODD) as a subset of the formulated OSS that is ``almost" forward invariant for the multi-agent dynamics. As the characterized domain does not intersect with the set of failure events, the SV is also almost safe in the given domain except for an arbitrarily small subset with a prescribed confidence level. The main contribution of this paper is further summarized as follows.

\textbf{An operational domain specific safety indicator} The proposed method characterizes an operational domain specific set using $\alpha$-shape~\cite{akkiraju1995alpha} and other coverage properties, which formally answers question Q1. The effectiveness of the proposed methodology is empirically demonstrated through a group of challenging cases. The study not only includes the classic three-dimensional lead-vehicle following domain, but also considers the challenging vehicle-to-vehicle and vehicle-pedestrian interactions with up to a 17-dimensional state space.

\textbf{An unbiased safety indicator} The $\epsilon$-almost robustly forward invariant set~\cite{weng2021towards,weng2021formal} is a provably unbiased safety indicator that generalizes the observation from sampled driving data to the incompletely observed operable domain. In particular, given a certain confidence level, the probability coefficient $\epsilon$ answers question Q2 by provably quantifying the performance of the SV statistically within the constructed set. The process does not involve any asserted behavioral assumptions, distribution estimates, or model fitting.

\textbf{Empirical evaluation} The empirical performance of the proposed method is demonstrated in a series of cases covering a variety of fidelity levels (real-world and simulators), driving environments (highway, urban, and intersections), road users (car, truck, and pedestrian), and SV driving behaviors (human driver and self driving algorithms). 

\subsection{Constructions and Notation}

\textbf{Notation: } The set of real and positive real numbers are denoted by $\R$ and $\R_{>0}$, respectively. $\Z$ denotes the set of all positive integers and $\Z_N=\{1,\ldots,N\}$. The $\ell_{\infty}$-norm is denoted by $\norm{\cdot}$. $|\mathcal{X}|$ is the cardinality of the set $\mathcal{X}$, e.g., for a finite set $\mathcal{D}$, $|\mathcal{D}|$ denotes the total number of points in $\mathcal{D}$. Let $\partial \Phi$ be the boundary of the set $\Phi \subseteq \R^n$. Some commonly adopted acronyms are also adopted including i.i.d. (independent and identically distributed), w.r.t. (with respect to), and w.l.o.g. (without loss of generality). 

In the remainder of the paper, Section~\ref{sec:preliminaries} will present the preliminaries along with formulating the \emph{finite-sampling operable domain quantification problem}. Section~\ref{sec:main} introduces details of the proposed safety metric. The empirical performance of the proposed metric is demonstrated in Section~\ref{sec:case_study}. Section~\ref{sec:conclusion} summarizes the paper and discusses future work of interest.

\section{Preliminaries and Problem Formulation}\label{sec:preliminaries}

\subsection{Mixed-Traffic Environment Formulation}
\begin{figure*}[t]
    \centering
    \includegraphics[width=0.99\linewidth]{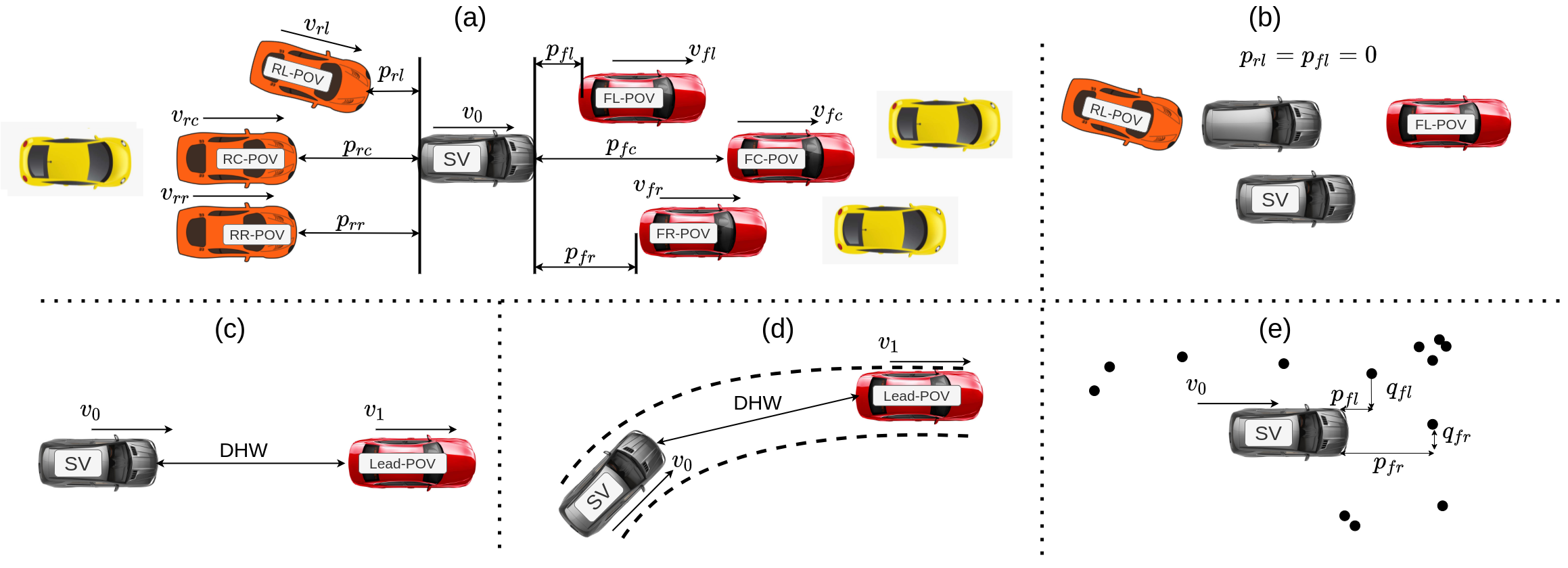}
    \caption{\small{Some illustrative examples of OSS specifications considered by this study: (a) a 13-dimensional multi-vehicle interactive state space of near-by vehicle-only traffics defined over the SV's local coordinates, (b) a special case to complement the ODD specifications in (a), (c) a 3-dimensional lead-vehicle following state space, (d) a generalization of (c) to non-straight road-segments, (e) a 5-dimensional vehicle-pedestrian interactive state space.}}
    \label{fig:model}
\vspace{-5mm}
\end{figure*}
Consider the mixed traffic environment as a time-variant heterogeneous multi-agent system of $k(t)$ agents at time $t$ where the $i$-th ($i\in\Z_{k(t)}$) agent admits the motion dynamics as 
\begin{equation}\label{eq:dyn}
    \s_i(t+1) = f_i(\s_i(t);\boldsymbol\omega_i(t))
\end{equation}
with state $\s_i \in \R^{n_i}$, disturbances and uncertainties $\boldsymbol\omega_i \in \R^{\omega_i}$, $f_i: \R^{n_i} \times \R^{\omega_i} \rightarrow \R^{n_i}$. Note that the agent is not limited to dynamic road users (e.g. vehicles, pedestrians, cyclists), but can also include other environmental features such as the traffic light color, stop sign position, weather condition, and road surface friction. Let the index $i=0$ denote the test SV. For a fleet of SVs, one can assign the index 0 to each SV iteratively for further analysis as the safety of each individual SV ensures the overall safety of the fleet. In general, the above system can be very complex as the real-world driving environment has very large $k$ and varies with respect to time.

The desired \emph{Operational Design Domain} (ODD) is thus introduced to specify the set within which the SV is expected to operate safely. Formal specifications of the ODD is further derived from an OSS with explicitly defined observable states $\s \in \mathcal{S} \subseteq \R^n$ and implicitly induced disturbances and uncertainties $\boldsymbol\omega \in \Omega \in \R^\omega$. Some examples of OSS specifications are presented in Fig.~\ref{fig:model}. 
This paper is primarily focused on three OSS specifications that are explained as follows.

\subsubsection{The lead-vehicle following domain}
This OSS characterizes the lead-vehicle following safety performance. It incorporates all instances from the on-road driving data with a preceding vehicle presented in the same lane with the SV. It is applicable for many ADAS features such as Automatic Emergency Braking (AEB) and Traffic-Jam Assist (TJA). The lead-vehicle following domain is also a commonly studied instance with other domain specifications incorporating time duration~\cite{arief2021deep} and assumed hybrid control modes~\cite{fan2017d}. In this paper, we consider a more general configuration than the aforementioned references, as the state specification $\mathcal{S}$ takes the speed of both vehicles ($v_0\in [0, v_{\max}]\subseteq \R_{\geq0}, v_1 \in [0, v_{\max}]\subseteq \R_{\geq0}$) and the bumper-to-bumper distance headway (DHW) ($p \in [0, p_{\max}]\subseteq \R_{\geq0}$) between the two vehicles as the states of interest. The specification is applicable for both straight road segments (Fig.~\ref{fig:model}(c)) and curved roads (Fig.~\ref{fig:model}(d)), i.e., the road curvature is considered to be part of $\Omega$, as are other factors such as road gradients, weather condition, and road surface friction. 

\subsubsection{The multi-vehicle interaction domain}
This OSS defines the SV's interaction with nearby vehicles. All the position states are represented with respect to the SV's local coordinates. The near-SV region is divided into 6 subregions: front-left (fl), front-center (fc), front-right (fr), rear-left (rl), rear-center (rc), and rear-right (rr). The left, the center, and the right regions are typically determined by the lane width. Within each region, the nearest vehicle is determined through the center-to-center $\ell_2$-norm distance against the SV. Two features of the nearest vehicles are selected by $\mathcal{S}$ including the bumper-to-bumper longitudinal distance clearance $p \in [p_{\min}, p_{\max}]\subseteq \R_{\geq0}$ against SV and the vehicle speed $v\in [0, v_{\max}]\subseteq \R_{\geq0}$. When presented with an alongside vehicle (i.e., part of the vehicles are overlapping longitudinally) on either side of the SV, the bumper-to-bumper distances are set to zero as shown in Fig.~\ref{fig:model}(b). Combined with the SV's speed $v_0$, we have a 13-dimensional state space, i.e., $\mathcal{S}\subset \R^{13}$. If a particular subregion is empty or if any of the states falls outside the domain of interest defined by $p_{\min}, p_{\max}, v_{\max}$ and other given bounds, a fixed low-risk state is assigned (e.g., if the front-center region is empty, we assign $p_{fc}=p_{\max}$ and $v_{fc}=v_0$). To have a valid state $\s\in\mathcal{S}$, at least one of the six subregions must remain non-empty with a vehicle satisfying the state bounds. The lateral distances are treated as uncertainties as each subregion is limited by the lane width, which already provides certain sideways localization information. Some other examples of disturbances and uncertainties include the presence of other dynamic road users, the road curvature, and different road infrastructures. A similar multi-vehicle configuration is also adopted by other studies for scenario extraction purposes~\cite{hauer2020clustering} and driver behavioral modeling~\cite{yan2021distributionally}.

\subsubsection{The vehicle-pedestrian interaction domain}
This OSS is primarily concerned with the SV interacting with pedestrians. Only pedestrians in front of the SV are involved in the specification due to responsibility oriented causes~\cite{shalev2017formal}. The front-left corner and the front-right corner of the SV are the reference points. Then, the nearest pedestrian to each reference point in terms of $\ell$2-norm distance is considered as the pedestrian of interest. For each pedestrian of interest, its longitudinal offset $p\in[0, p_{\max}]\subseteq \R_{\geq0}$ and lateral offset $1\in[0, q_{\max}]\subseteq \R_{\geq0}$ from the corresponding reference point are selected as part of the states in $\mathcal{S}$. Combined with the SV's velocity $v_0\in [0, v_{\max}]\subseteq \R_{\geq0}$, we have the 5-dimensional state space for the vehicle-pedestrian interaction domain.

Note that the above OSSs and possibly other variants can be further combined to formulate various mixed traffic operational environments. For example, combining the multi-vehicle interaction domain with the vehicle-pedestrian interaction domain, results in a $17$-dimensional state space. Moreover, the presented three OSSs are concrete examples that facilitate the metric proposal. In the general CAV applications, the OSS design could involve other explicit states such as the communication quality, road curvature, weather condition, and driver intention, to name a few. This requires additional capabilities in simulation or the collection of additional classes of data during real-world tests in order to observe the relevant states. This does not affect the execution or the theoretical properties of the proposed metric, but may lead to more extensive safety analysis outcomes of the SV driving algorithm.

\begin{remark}
The driving data studied in this paper can be collected from both on-road tests and scenario-based tests, as long as the data collection follows the nominal distribution of the mixed-traffic driving environment within which the SV is being tested. 
\end{remark}

W.l.o.g., let there be some states from the collected driving data consistent with the given $\mathcal{S}$. We then have the primary driving data $\mathcal{D}\subseteq \mathcal{S}$ which comprises finite observations that allow us to implement the safety analysis of the SV's performance in $\mathcal{S}$. Moreover, some of the states are consecutively collected in time w.r.t. the same SV, which is further referred to as a \emph{trajectory} $\tau \subseteq \mathcal{D}$. In this paper, we often extract state transition pairs from all trajectories in $\mathcal{D}$ as $\mathcal{TD}=\{\s_i, \s'_i\}_{i=1,\ldots,M} \subseteq \mathcal{S}^2$, where $\s_i, \s'_i$ are two consecutively collected states in time w.r.t. the same SV. Note that each state $\s \in \mathcal{S}$ also inherently admits a certain motion dynamics as
\begin{equation}\label{eq:s-dyn}
    \s'(t)=\s(t+1) = f(\s(t); \boldsymbol\omega(t)), \s \in \mathcal{S}, \boldsymbol\omega \in \Omega, f: \mathcal{S}\times \Omega \rightarrow \mathcal{S}.
\end{equation}

In this paper, our focus is to present a safety performance evaluation metric that identifies the real operable domain in a data-driven manner (from $\mathcal{D}$) and identifies its safety-related properties. This is formally presented as the \textit{finite-sampling operable domain quantification problems}, as we shall introduce in the following section.
 
\subsection{Finite-sampling Safety Assurance with Set Invariance}
Given a certain set $\Phi$, if one continuously observes sampled state transitions staying inside $\Phi$, then the set $\Phi$ is potentially forward invariant. To formally quantify such a statistical potential, we introduce the almost forward invariant set adapted from~\cite{weng2021towards} as follows.

\begin{definition}{[\textbf{$\epsilon$-Almost Robustly Forward Invariant Set}]}\label{def:almost-rfiv}
    Let $\epsilon \in (0,1)$, the set $\Phi \subseteq \mathcal{S}$ is $\epsilon$-almost robustly forward invariant for \eqref{eq:s-dyn} if
    \begin{equation}
        \forall \s \in \mathcal{S}, \forall \boldsymbol\omega \in \Omega, \mathbb{P}(\{f(\s, \boldsymbol\omega) \notin \mathcal{S}\}) \leq \epsilon.
    \end{equation}
\end{definition}

To further relate the above definition to the purpose of safety analysis, let $\mathcal{C} \subset \mathcal{S}$ be the set of unsafe states such as collisions. Then we have the following definition for the almost safe set.
\begin{definition}{[\textbf{$\epsilon$-Almost Safe Set}]}\label{def:almost-safe}
    Let $\epsilon \in (0,1)$, the SV is $\epsilon$-almost safe in $\Phi \subseteq \mathcal{S}$ if $\Phi \cap \mathcal{C}=\emptyset$ and $\Phi$ is $\epsilon$-almost robustly forward invariant for \eqref{eq:s-dyn} by Definition~\ref{def:almost-rfiv}.
\end{definition}

The problem of interest for this paper is than formally presented as follows.
\begin{problem}{[\textbf{The Finite-Sampling Operable Domain Quantification Problem}]}\label{prob}
    Given $\beta \in (0,1)$ and a group of observed states $\mathcal{D} \subseteq \mathcal{S}$, the finite-sampling operable domain quantification problem seeks an algorithm $\text{ALG}: \mathcal{S} \times (0,1) \rightarrow \mathcal{S} \times (0,1)$ that identifies a certain set $\Phi \subseteq \mathcal{S}$ and $\epsilon\in(0,1)$ such that with confidence level of at least $1-\beta$, the SV is $\epsilon$-almost safe in $\Phi$ by Definition~\ref{def:almost-safe}.
\end{problem}

The above problem formulation is fundamentally different from many of the existing CAV safety metrics as mentioned in Section~\ref{sec:introduction}. The desired set $\Phi$ is the specific operable domain within which the SV is expected to operate safely. The probability coefficient $\epsilon$ quantifies the statistical potential of the SV's safety performance in $\Phi$. The next section discusses details of the proposed algorithm that solves the aforementioned problem. 

\section{Finite-Sampling Operable Domain Quantification}\label{sec:main}
The proposed solution to Problem~\ref{prob} follows a two-step procedure in general including (i) set construction and (ii) set validation. The set construction step seeks to construct a certain set $\Phi$ from $\mathcal{D}$. Second, as one replays data in $\mathcal{D}$, one can then validate the almost forward invariance property of the constructed set $\Phi$ through consecutively observing transitions among states in $\Phi$. The derived $\epsilon$ also relies on the given confidence level defined in Problem~\ref{prob}. For the remainder of this section, we shall address the aforementioned two steps, respectively, in Section~\ref{subsec:construction} and Section~\ref{subsec:validation}. The complete algorithm is summarized in Section~\ref{subsec:alg}.

\subsection{Set Construction with $\alpha$-Shape and Coverage Measures}\label{subsec:construction}
For the safety evaluation purpose, the constructed set is expected to cover all potentially safe points. This is referred to as $\mathcal{D}_s$ and is obtained through Algorithm~\ref{alg:D_s}. A series of methods is then proposed to formally construct a set that characterizes the \emph{shape} and the \emph{coverage} information of $\mathcal{D}_s$.

\begin{algorithm}[H]
    \begin{algorithmic}[1]
    \State {\bf Input:} $\mathcal{D}\subseteq \mathcal{S}, \mathcal{C}$
    \State {\bf Initialize: } Empty graph $G_s=(\mathcal{D}_s, E_s), \mathcal{D}_s\subseteq\mathcal{S}, E_s\subseteq \mathcal{S}^2, i=1$
    \State {Collect all safe trajectories $\{\tau^s_i\}_{i\in\Z_s}, \tau^s_i \subseteq\mathcal{D}, \tau^s_i\!\cap\!\mathcal{C} =\emptyset$}
    \State {Collect all unsafe trajectories $\{\tau^c_i\}_{i\in\Z_c}, \tau^c_i \subseteq\mathcal{D}, \tau^c_i\!\cap\!\mathcal{C} \neq\emptyset$}
    \State {{\bf While} $i<s$:}
    \State {\ \ \ \ $j=1$}
    \State {\ \ \ \ {\bf While} $j<|\tau^s_i|-1$:}
    \State {\ \ \ \ \ \ \ \ $G_s$.\texttt{add}$((\tau^s_i[j],\tau^s_i[j+1]))$}
    \State {\ \ \ \ \ \ \ \ $j+=1$}
    \State {\ \ \ \ $i+=1$}
    \State {{\bf If} $\mathcal{D}\cap\mathcal{C}\neq \emptyset$}
    \State {\ \ \ \ $i=1$}
    \State {\ \ \ \ {\bf While} $i<c$:}
    \State {\ \ \ \ \ \ \ \ {\bf For} $\s$ in $\tau^c_i$ {\bf do}}
    \State {\ \ \ \ \ \ \ \ \ \ \ \ $R_s=$\texttt{Reachable}$(\s, G_s)$}
    \State {\ \ \ \ \ \ \ \ \ \ \ \ $G_s$.\texttt{remove}$(R_s)$}
    \State {\ \ \ \ \ \ \ \ {\bf End For}}
    \State {\ \ \ \ $i+=1$}
    \State {{\bf End If}}
    \State {\bf Output:} $\mathcal{D}_s$
    \end{algorithmic}
    \caption{Extract all potential safe states from $\mathcal{D}$} \label{alg:D_s}
\end{algorithm}

Note that \texttt{Reachable}$(\s, G_s)$ returns all vertices on the graph $G_s$ that connects, directly or indirectly, to the given point $\s$. In practice, this is achieved through a standard depth-first-search (DFS) routine. Moreover, \texttt{add}$()$ and \texttt{remove}$()$ are both notation functions where $G_s$.\texttt{add}$((\s,\s'))$ adds the edge $(\s,\s')$ to the graph $G_s$, and $G_s$.\texttt{remove}$(R_s)$ removes all vertices in $\mathcal{D}_s\cap R_s$ from $G_s$.

We are now ready to construct the potentially safe set $\Phi$ from $\mathcal{D}_s$. In this paper, we adopt the \emph{$\alpha$-shape}~\cite{akkiraju1995alpha} to characterize the shape of the desired set. The following definition is standard~\cite{alphashape2011}.
\begin{definition}\label{def:alpha_shape}
    Consider a finite set of points $\mathcal{A}\subset \R^n$. Let an $\alpha$-ball be an open ball with radius $\alpha\in\R_{\geq0}$. Let $\Delta_{\Psi}$ be a $k$-simplex for some $\Psi \subset\mathcal{A}$ such that $|\Psi|=k+1$. A $k$-simplex is $\alpha$-exposed if there exists an empty $\alpha$-ball $b_{\alpha}$ with $\Psi=\partial b_{\alpha} \cap \mathcal{A}$. An $\alpha$-shape, $\Phi_{\alpha}(\mathcal{A})$, of the set $\mathcal{A}$ satisfies $\mathcal{A} \subset \Phi_{\alpha}(\mathcal{A})$ and 
    \begin{equation}
        \partial \Phi_{\alpha}(\mathcal{A})=\{ \Delta_{\Psi} \mid \Psi\subset \mathcal{A}, |\Psi|\leq n, \Delta_{\Psi} \text{ is } \alpha-\text{exposed} \},
    \end{equation}
    i.e., the boundary of the $\alpha$-shape consists of all $k$-simplices of $\mathcal{A}$ for $0\leq k < n$ which are $\alpha$-exposed.
\end{definition}
It follows that $\lim_{\alpha\rightarrow0}\Phi_{\alpha}(\mathcal{A})=\mathcal{A}$ and $\Phi_{\infty}(\mathcal{A})$ is the ordinary convex hull of $\mathcal{A}$. The $\alpha$-shape of a finite point set $\mathcal{A}$ is uniquely determined by $\mathcal{A}$ and $\alpha$. For any given $\alpha$, the corresponding $\alpha$-shape determination algorithm comes with a time complexity of $O( N \log{N} )$~\cite{akkiraju1995alpha}, where $N$ denotes the number of points in $\mathcal{A}$, i.e., $N=|\mathcal{A}|$. In practice, one may also require a certain preferred shape such as a single polytope that wraps $\mathcal{A}$ with the smallest cardinality. This implies a certain cost function of the $\alpha$-shape with the optimal cost determined by the preferred shape. This is typically performed through a logarithmic search of $\alpha$-shapes by modifying the lower and upper bounds of the tested $\alpha$ until the gap between the two bounds becomes sufficiently small~\cite{kengithub}. However, with this method the computational complexity also increases.

\begin{remark}
    In the previous literature of scenario-sampling almost safe set validation~\cite{weng2021towards,weng2021formal}, the $\delta$-covering set is adopted to characterize the set construction. However, as indicated by Problem~\ref{prob}, the data set is presented as it is in this study, and one cannot control the scenario sampling to modify the testing procedure or to add more observations for analysis. For sparse data sets, the $\delta$-covering set tends to have a significant over-approximation. Moreover, the $\delta$-covering set of the finite set is not unique for all non-zero $\delta$s. As a result, the $\alpha$-shape is a more flexible solution to handle various levels of data sparsity with a uniquely determined solution for a given finite set.
\end{remark}

For this paper, the finite set being considered is $\mathcal{D}_s$ derived from Algorithm~\ref{alg:D_s}. Hence the $\alpha$-shape takes the notation as $\Phi_{\alpha}(\mathcal{D}_s)$. Note it is also expected that $\Phi_{\alpha}(\mathcal{D}_s) \cap (\mathcal{D}\setminus \mathcal{D}_s) = \emptyset$. This may require exploring different choices of $\alpha$ if a certain given $\alpha$ fails to satisfy the condition. In practice, this is a rare case as $\mathcal{D}_s$ and $\mathcal{D}\setminus\mathcal{D}_s$ are typically easily distinguishable through a certain set of piece-wise linear constrains. None of the examples studied in Section~\ref{sec:case_study} required this additional suggested exploration.

Finally, to characterize the coverage performance of $\mathcal{D}_s$, we also adopt the following two measures to characterize the density $D$ and occupancy $O$ as
\begin{equation}\label{eq:coverage}
    D=\frac{|\mathcal{D}_s|}{|\Phi_{\alpha}(\mathcal{D}_s)|} \text{ and } O=\frac{|\Phi_{\alpha}(\mathcal{D}_s)|}{|\mathcal{S}|}.
\end{equation}
Note that there also exists other coverage indicators such as the index of dispersion~\cite{selby1965index} and the star discrepancy~\cite{dang2008sensitive}. However, the index of dispersion is not directly applicable in this paper as $\mathcal{S}$ is not all positive. The star discrepancy is not selected for computational complexity concerns. Other representative coverage metrics are of future interest.

\subsection{Finite-Sampling Almost Robustly Forward Invariant Set Validation}\label{subsec:validation}
We are now ready to characterize how safe the SV is in $\Phi_{\alpha}(\mathcal{D}_s)$. Suppose the validation were to be executed online with a single SV. The data acquisition of $\mathcal{D}$ and its corresponding $\mathcal{TD}$ are thus collected following a particular time sequence w.r.t. the same SV. At a certain step, if one starts consecutively observing transitions that start and stay inside $\Phi_{\alpha}(\mathcal{D}_s)$ until the end of the test, one then has statistical evidence to claim the robustly forward invariance property of $\Phi_{\alpha}(\mathcal{D}_s)$ by Definition~\ref{def:almost-rfiv}. This is presented as a validation routine in Algorithm~\ref{alg:conutN}.

\begin{algorithm}[H]
    \begin{algorithmic}[1]
    \State {\bf Input:} A set of state transitions $\mathcal{TD}\subseteq\mathcal{S}^2$, $\Phi \subseteq \mathcal{S}$
    \State {\bf Initialize: } $i=1, N=0$
    \State {{\bf While} $i<|\mathcal{TD}|$:}
    \State {\ \ \ \ $(\s, \s') = \mathcal{TD}[i]$}
    \State {\ \ \ \ {\bf If} $\s \in \Phi$ and $\s'\in\Phi$:}
    \State {\ \ \ \ \ \ \ \ $N+=1$}
    \State {\ \ \ \ {\bf Else}}
    \State {\ \ \ \ \ \ \ \ $N=0$}
    \State {\ \ \ \ {\bf End If}}
    \State {\ \ \ \ $i += 1$}
    \State {\bf Output:} $N$
    \end{algorithmic}
    \caption{Count consecutive safe transitions for validation $\mathcal{VAL}(\mathcal{TD}, \Phi)$ } \label{alg:conutN}
\end{algorithm}

However, the above described online procedure is not applicable for a fleet of SVs deployed in the nominal driving environment test simultaneously. Moreover, a safety metric is primarily used to analyze the safety performance of a system in a post-processing manner. That is, one replays the data set following a certain order of all elements in $\mathcal{TD}$. For statistical inference, as the set of initializations of all transition pairs are i.i.d. w.r.t. the underlying distribution on $\mathcal{S}$, $\mathcal{TD}$ can thus be replayed in any order. In particular, the replay of $\mathcal{D}$ is formally specified as follows.
\begin{definition}
    Consider the domain-specific finite set $\mathcal{D}$ and the corresponding set of all state transitions $\mathcal{TD}$ as presented in Section~\ref{sec:preliminaries}. The replay of $\mathcal{D}$, $\sigma(\mathcal{D})$, is a permutation of $\mathcal{TD}$, i.e., a certain rearrangement of all elements in $\mathcal{TD}$.
\end{definition}
It is immediate that the total number of possible replays of $\mathcal{D}$ is $|\mathcal{TD}|!$. As long as the probability for each replay order to occur remains the same (i.e., $1/|\mathcal{TD}|!$), the set of initialization of all transition pairs in $\mathcal{TD}$ remains i.i.d. w.r.t. the same underlying distribution on $\mathcal{S}$. We can then formally justify the safety performance of the SV in $\Phi_{\alpha}(\mathcal{D}_s)$ through the following theorem.
\begin{theorem}{[\textbf{$\bar{\epsilon}\alpha$-Almost Robustly Forward Invariance Validation}]}\label{thm:finite-sample-almost-safe}
    Consider $\mathcal{S}\subseteq \R^n$, $\beta\in (0,1), \alpha \in \R_{\geq0}$, the domain-specific finite set $\mathcal{D}$ and the corresponding set of all state transitions $\mathcal{TD}$. Let $\mathcal{D}_s$ be the set of potentially safe states extracted from $\mathcal{D}$ through Algorithm~\ref{alg:D_s}. Let $\Phi_{\alpha}(\mathcal{D}_s)$ be the $\alpha$-shape of $\mathcal{D}_s$ as specified by Definition~\ref{def:alpha_shape}. For a certain replay of $\mathcal{D}$ denoted by the index $i$ as $\sigma_i(\mathcal{D})$, let $N_i=\mathcal{VAL}(\sigma_i(\mathcal{D}), \Phi_{\alpha}(\mathcal{D}_s))$. Then, we have that $\Phi_{\alpha}(\mathcal{D}_s)$ is $\epsilon_i\alpha$-almost robustly forward invariant with confidence level $1-\beta$, and 
    \begin{equation}
        \epsilon_i \leq 1-\exp{\left(\frac{\ln{\beta}}{N_i}\right)}.
    \end{equation}
    Moreover, $\Phi_{\alpha}(\mathcal{D}_s)$ is expected to be $\bar{\epsilon}\alpha$-almost robustly forward invariant with confidence level $1-\beta$ and
    \begin{equation}\label{eq:expected_epsilon}
        \bar{\epsilon} = \mathbb{E}(\epsilon) = \sum_{i=1}^{|\mathcal{TD}|!} \frac{\epsilon_i}{|\mathcal{TD}|!}.
    \end{equation}
    As $\Phi_{\alpha}(\mathcal{D}_s) \cap \mathcal{C}=\emptyset$, $\Phi_{\alpha}(\mathcal{D}_s)$ is also an $\bar{\epsilon}\alpha$-almost safe set.
\end{theorem}
\begin{proof}
For any fixed choice of $\alpha \in \R_{\geq0}$ and a particular data replay $\sigma_i(\mathcal{D}), i\in \Z_{|\mathcal{TD}|!}$, the proof of the $\epsilon_i\alpha$-almost robustly forward invariance property is a direct outcome from Theorem 2 in~\cite{weng2021towards}. Furthermore, consider $\epsilon$ as a random variable and the occurrence probability for each replay is the same, i.e., $\mathbb{P}(\epsilon=\epsilon_i)=1/|\mathcal{TD}|!$. The expected $\epsilon$ is thus obtained in the form of~\eqref{eq:expected_epsilon}. Finally, the $\bar{\epsilon}\alpha$-almost safe property is a direct outcome of Definition~\ref{def:almost-safe}.
\end{proof}
\begin{figure}
    \centering
    \vspace{2mm}
    \includegraphics[width=0.49\textwidth]{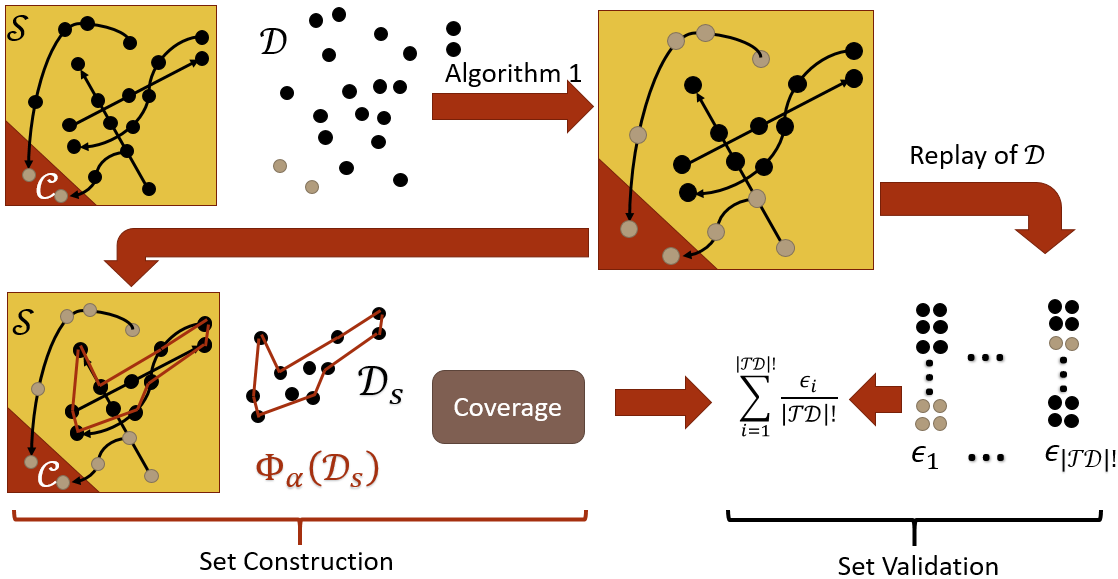}
    \caption{Overview of the proposed finite-sampling safe operable domain quantification algorithm.}
    \label{fig:main}
    \vspace{-3mm}
\end{figure}
Note that the derived forward invariance property essentially relies on the system model \eqref{eq:dyn}, which fundamentally assumes the motion dynamics follows the Markov Decision Process (MDP). This is not necessarily valid in practice as, for example, the SV driving algorithm may also rely on a series of historical state observations for decision-making and control. A more general analysis regarding the non-MDP configuration is beyond the scope of this paper.

\subsection{Finite-sampling Operable Domain Quantification Algorithm}\label{subsec:alg}

So far, we have presented the set construction and set validation steps. The complete algorithm that tackles Problem~\ref{prob} is summarized in Algorithm~\ref{alg:main} and is also conceptually illustrated in Fig.~\ref{fig:main}.
\begin{algorithm}[H]
    \begin{algorithmic}[1]
    \State {\bf Input:} $\mathcal{D}\subseteq \mathcal{S}, \mathcal{C}$, $\mathcal{TD}\subseteq\mathcal{S}^2$, $\alpha\in \R_{\geq0}, \beta\in(0,1)$
    \State {Determine $\mathcal{D}_s$ with Algorithm~\ref{alg:D_s}}
    \State {Determine $\Phi_{\alpha}(\mathcal{D}_s)$ given $\mathcal{D}_s$ and $\alpha$~\cite{akkiraju1995alpha,kengithub} }
    \State {Determine all $s$ pairs of state transitions $\mathcal{TD}_s\subseteq\mathcal{TD}$ such that $\mathcal{TD}_s \subseteq \mathcal{D}_s^2$ and $\mathcal{TD}\setminus\mathcal{TD}_s\subseteq (\mathcal{D}\setminus\mathcal{D}_s)^2$}
    \State {$\bar{\epsilon} = 0$}
    \State {{\bf For} $i$ in $\Z_{s}$ {\bf do}}
    \State {\ \ \ \ $\epsilon_i = 1-\exp{\left(\frac{\ln{\beta}}{i}\right)}$}
    \State {\ \ \ \ $p_i = \frac{(i!)\cdot(|\mathcal{TD}|-i)!}{|\mathcal{TD}|!}$}
    \State {\ \ \ \ $\bar{\epsilon}+=\epsilon_i p_i$}
    \State {{\bf End For}}
    \State {\bf Output:} $\Phi_{\alpha}(\mathcal{D}_s), \bar{\epsilon}$ 
    \end{algorithmic}
    \caption{Finite-sampling Operable Domain Quantification} \label{alg:main}
\end{algorithm}
Note that the derivation of $\bar{\epsilon}$ is slightly different from Theorem~\ref{thm:finite-sample-almost-safe} as replays sharing the same $\epsilon$ value are grouped together to improve the computational performance. The density and occupancy features can also be derived through~\eqref{eq:coverage}. Given the finite set $\mathcal{D}_s$ and the selected $\alpha$, $\Phi_{\alpha}(\mathcal{D}_s)$ is determined by the standard $\alpha$-shape algorithm~\cite{akkiraju1995alpha,kengithub}. In practice, we also use the discussed logarithmic search scheme in Section~\ref{subsec:construction} to determine the appropriate $\alpha$. Implementation details will also be discussed in Section~\ref{sec:case_study}. 

Most of the computational burden for Algorithm~\ref{alg:main} comes from the alpha-shape construction (line 3). For a fixed $\alpha$, the adopted algorithm has a computational complexity of $O(N_D\cdot N_{\alpha})$ with $N_D$ being the cardinality of $\mathcal{D}_s$ and $N_{\alpha}$ being the cardinality of the finite set of points that form the shape of $\Phi_{\alpha}(\mathcal{D}_s)$. In general, $N_{\alpha}$ is often significantly smaller than $N_D$

We conclude this section by emphasizing that the $\mathcal{D}_s$ and the obtained $\alpha$-shape are not only embedded with coverage and forward invariance information. The graph $G_s$ induces state transitions that could be used for other safety related applications such as fault tree analysis with backtracking process algorithms~\cite{hejase2020methodology, capito2021bpa} and information gain justification~\cite{collin2021plane}. The states can also be associated with other safety features available from the raw data such as human driver engagement (e.g., a human may tend to engage within a certain subset of the obtained covering set) and system signals (e.g., the forward collision warning may only be triggered in a certain subregion). Existing driving data sets collected from real-world and simulators are not comprehensive enough to provide the aforementioned features. Hence, considering those features is regarded as future work.

\section{Case Study} \label{sec:case_study}
To demonstrate the performance of the proposed safety metric, a series of cases are studied in this section. Detailed configurations are summarized in Table~\ref{tab:summary} and explained as follows.

\begin{table*}[]
    \vspace{2mm}
    \centering
    \resizebox{0.95\textwidth}{!}{%
    \begin{tabular}{|c|l|c|c|c|c|c|c|c|c|}
    \hline
    \multicolumn{2}{|c|}{\textbf{Case}}                                                                     & \multicolumn{2}{c|}{HighD data}                                                              & Waymo open data set                                                                            & \multicolumn{2}{c|}{SUMO}                   & Carla     & \multicolumn{2}{c|}{NCAP-AEB}                                                                     \\ \hline
    \multicolumn{2}{|c|}{\textbf{\begin{tabular}[c]{@{}c@{}}Nominal\\ Driving \\ Environment\end{tabular}}} & \multicolumn{2}{c|}{\begin{tabular}[c]{@{}c@{}}real-world\\ (straight highway)\end{tabular}} & \begin{tabular}[c]{@{}c@{}}real-world\\ (highway, \\ urban, \\ intersection)\end{tabular}     & \multicolumn{3}{c|}{\begin{tabular}[c]{@{}c@{}}simulation\\ (highway, urban, intersection)\end{tabular}}      & \multicolumn{2}{c|}{\begin{tabular}[c]{@{}c@{}}simulation\\ (straight highway)\end{tabular}}             \\ \hline
    \multicolumn{2}{|c|}{\textbf{SV Driver}}                                                                & \multicolumn{2}{c|}{human}                                                                   & \multirow{2}{*}{\begin{tabular}[c]{@{}c@{}}Waymo's self-driving \\ car (SDC)\end{tabular}} & \multicolumn{2}{c|}{lane change heuristics} & \multirow{2}{*}{\begin{tabular}[c]{@{}c@{}}Carla \\ autopilot\end{tabular}} & &\\ \cline{6-7}
    \multicolumn{2}{|c|}{}                                                                                  & \multicolumn{2}{c|}{driver}                                                                  &                                                                                               & IDM\_0                & IDM\_1              &                                                  &    IDM\_0                & IDM\_1                        \\ \hline
    \multicolumn{2}{|c|}{\textbf{SV Type}}                                                                  & Car                                           & Truck                                        & Car                                                                                           & \multicolumn{2}{c|}{Car}                    & Car                                                   & \multicolumn{2}{c|}{Car}                       \\ \hline
    \multicolumn{2}{|c|}{\textbf{\begin{tabular}[c]{@{}c@{}}Other \\ Traffic Type\end{tabular}}}            & \begin{tabular}[c]{@{}c@{}}Car \\ \& Truck\end{tabular} &\begin{tabular}[c]{@{}c@{}}Car \\ \& Truck\end{tabular}                                          & \begin{tabular}[c]{@{}c@{}}Car \\ \& Pedestrian\end{tabular}                                  & \multicolumn{2}{c|}{Car}                    & \begin{tabular}[c]{@{}c@{}}Car \\ \& Pedestrian\end{tabular}   & \multicolumn{2}{c|}{Car}               \\ \hline
    \end{tabular}%
    }
    \caption{Summary of all real-world data sets and simulators used for the case study section.}
    \label{tab:summary}
    \vspace{-5mm}
\end{table*}

\paragraph{HighD data set } The HighD data set~\cite{krajewski2018highd} is a data set of naturalistic vehicle trajectories recorded on German highways. The data set comes with a mixture of car and truck drivers operating on straight-road highway segments. It is a well-known fact that naturalistic driving behavior exhibits statistical consensus in general, but also with discrepancies that depend on the vehicle type. This inspires our study in this section by analyzing the human driver safety performance w.r.t. different vehicle types and different ODDs. 
\paragraph{Waymo open data} The Waymo open data~\cite{sun2020scalability} used in this study is the motion data set, which is primarily used for training and validating traffic motion prediction algorithms. In this study, we redirect the data set to the safety analysis purpose by taking advantage of the motion trajectories recorded for Waymo's self-driving car (SDC) and the surrounded mixed-traffic road users, especially vehicles and pedestrians.
\paragraph{SUMO} The Simulation of Urban MObility (SUMO)~\cite{SUMO2018} is an open source, microscopic, and continuous multi-modal traffic simulator. In this study, we compare two parametric self-driving algorithms, referred to as IDM\_0 and IDM\_1, both based on the Intelligent Driving Model (IDM)~\cite{treiber2013traffic}, a well-known and well-understood vehicle following controller, as an illustration of the application of our metric. The simulated traffic is created with a variety of vehicles of different dynamics and self-driving configurations operating in a mapped environment with a mixture of highway and urban roads. A fleet of 20 SVs with each parametric policy is then deployed along with the simulated traffic for a 24-hour operation in simulation. 
\paragraph{Carla} The simulated mixed-traffic environment in Carla~\cite{dosovitskiy2017carla} involves a fleet of vehicles driven by the default autopilot algorithm along with randomly deployed other vehicles and pedestrians. The simulation is executed in four mapped environments (Town 01 to 04) in Carla. Within each environment, the simulation terminates when the SV reaches a total travel distance of 100-kilometer.
\paragraph{NCAP-AEB} The SV algorithms adopted to create this data set are the same as those in the SUMO case. Two parametric IDM algorithms are deployed in a simulated straight-road segment with the lead principal other vehicle (POV) executing the testing policy specified in the NCAP Autonomous Emergency Braking (AEB) car-to-car test program~\cite{van2017euro}. The program involves 48 testing scenarios with each scenario executed once.

\begin{remark}\label{rmk:idms}
    To differentiate between IDM\_0 and IDM\_1, IDM\_0 is parameterized with a stronger braking capability but has a higher tolerance of short following distances (due to a smaller minimum safe distance and a smaller safe time headway as shown in Table~\ref{tab:idm_config}). The IDMs used in NCAP-AEB and SUMO mostly share similar specifications. However, the SUMO simulator also includes other hyper-parameters that may affect the performance, such as the perfectness and the lateral lane change heuristics, which have been kept with their default values. 
\end{remark}

\begin{table}
    \vspace{2mm}
    \centering
    \resizebox{0.45\textwidth}{!}{%
    \begin{tabular}{l|cc}
    \hline
    SV                                         & \multicolumn{1}{l|}{IDM\_0} & \multicolumn{1}{l}{IDM\_1} \\ \hline
    minimum safe distance (m)                  & \multicolumn{1}{c|}{0.5}    & 4                          \\ \hline
    safe time headway (s)                      & \multicolumn{1}{c|}{0.1}    & 4                          \\ \hline
    maximum brake control $\mathrm{m/s^2}$)    & \multicolumn{1}{c|}{9}      & 2                          \\ \hline
    free-traffic speed $\mathrm{m/s}$)         & \multicolumn{2}{c}{25}                                   \\ \hline
    maximum acceleration $\mathrm{m/s^2}$)     & \multicolumn{2}{c}{0.73}                                 \\ \hline
    comfortable deceleration $\mathrm{m/s^2}$) & \multicolumn{2}{c}{1.67}                                 \\ \hline
    exponent of acceleration                   & \multicolumn{2}{c}{4}                                    \\ \hline
    \end{tabular}%
    }
    \caption{The hyper-parameter configurations of the two IDM variants for the SUMO and NCAP-AEB cases.}
    \label{tab:idm_config}
    \vspace{-5mm}
\end{table}

\begin{remark}
    As a pure data-driven approach, the obtained safety performance evaluations throughout this section are only based on the given data assuming the collected data points are i.i.d. w.r.t. the distribution in the nominal driving environment. This is generally true in simulator-based tests such as in Carla and SUMO, but is not necessarily valid for real-world driving data sets as the data processing details are largely unknown. As a result, the claimed safety performance from the HighD data set and the Waymo open motion data set do not necessarily represent the corresponding SVs' actual safety performance.
\end{remark}

Before proceeding to the domain-oriented safety evaluation outcomes, we first emphasize some featured observations:
\begin{itemize}
    \item The safe operable domain of a certain SV is a joint outcome of the SV's own driving behavior, the other dynamic road users' behavior, and the test environment.
    \item Within the same case study (i.e., the same testing behavior and environment), it is in general inaccurate to claim that a certain SV is safer than the other, unless the outcome concurs among all features, i.e., small $\bar{\epsilon}$, large density, and large occupancy.
    \item Comparing the proposed safety metric with the statistical fatality rate inference~\cite{fraade2018measuring}, given the same confidence level and the same data set $\mathcal{D}$, the magnitude of $\bar{\epsilon}$ is significantly smaller than the fatality rate value. That is, the inferred fatality rate metric tends to over-estimate the risk, especially when the collected finite states are clustered in a specific sub-domain in the nominal driving environment. In the meanwhile, the operational domain specific nature of the proposed metric helps establish a more precise safety performance assessment. 
\end{itemize}

\begin{table*}
    \vspace{2mm}
    \centering
    \resizebox{0.95\textwidth}{!}{%
    \begin{tabular}{c|c|c|c|c|c||c|c|c}
    \hline
    \begin{tabular}[c]{@{}c@{}}Data\\ Source\end{tabular} & SV                          & \begin{tabular}[c]{@{}c@{}}Safe Distance\\ (km)\end{tabular} & 1-R(C=0.999)                & TTC (s)                                & TTC Valid Rate             & $\bar{\epsilon} (\times 10^{-4})$             & $\frac{|\mathcal{D}_s|}{|\Phi_{\alpha}(\mathcal{D_s})|}$                 & $\frac{|\Phi_{\alpha}(\mathcal{D_s})|}{|\mathcal{S}|}$ \\ \hline
    \multirow{2}{*}{NCAP-AEB}                             & IDM\_0                      & N/A                                 & N/A             & 1.368 $\pm$ 1.051         & {\bf 0.964}       & {\bf 13.0199}     & {\bf 0.832}         & 0.1960                                               \\
                                                          & IDM\_1                      & N/A                                 & N/A             & {\bf 1.229} $\pm$ 2.223   & 0.244             & 4.9077            & 2.464               & {\bf 0.0568}                                               \\ \hline
    \multicolumn{1}{c|}{\multirow{2}{*}{SUMO}}            & IDM\_0                      & 5725.99                             & 0.0019          & 8.844 $\pm$ 0.685         & 0.378             & {\bf 0.7717}      & {\bf 1.752}         & {\bf 0.2619}                                               \\
    \multicolumn{1}{c|}{}                                 & IDM\_1                       & N/A                                & N/A             & {\bf 8.581} $\pm$ 1.288   & {\bf 0.447}       & 0.5121            & 2.009               & 0.3265                                               \\ \hline
    \multirow{2}{*}{HighD}                                & Car                         & 3276.48                             & 0.0034          & {\bf 8.871} $\pm$ 0.666   & {\bf 0.507}       & 0.5732            & 7.675               & 0.5807                                               \\
                                                          & Truck                       & \textbf{551.81}                     & {\bf 0.0199}    & 8.951 $\pm$ 0.413         & 0.442             & {\bf 2.9755}      & {\bf 1.968}         & {\bf 0.4418}                                               \\ \hline
    \end{tabular}%
    }
    \caption{Safety study of the lead-vehicle following domain with the NCAP-AEB test of two parametric IDMs, two types of human-driven vehicles from the HighD data set, and two different parametric driving algorithms from the SUMO simulator. Bold typeface highlights values that indicate the higher-risk driving behavior.}
    \label{tab:lead_follow}
\vspace{-7mm}
\end{table*}
\begin{figure*}[!h]
\vspace{2mm}
\centering
\begin{subfigure}{0.32\linewidth}
  \centering
  \includegraphics[trim={2cm 1cm 3cm 1cm},clip,width=0.95\textwidth]{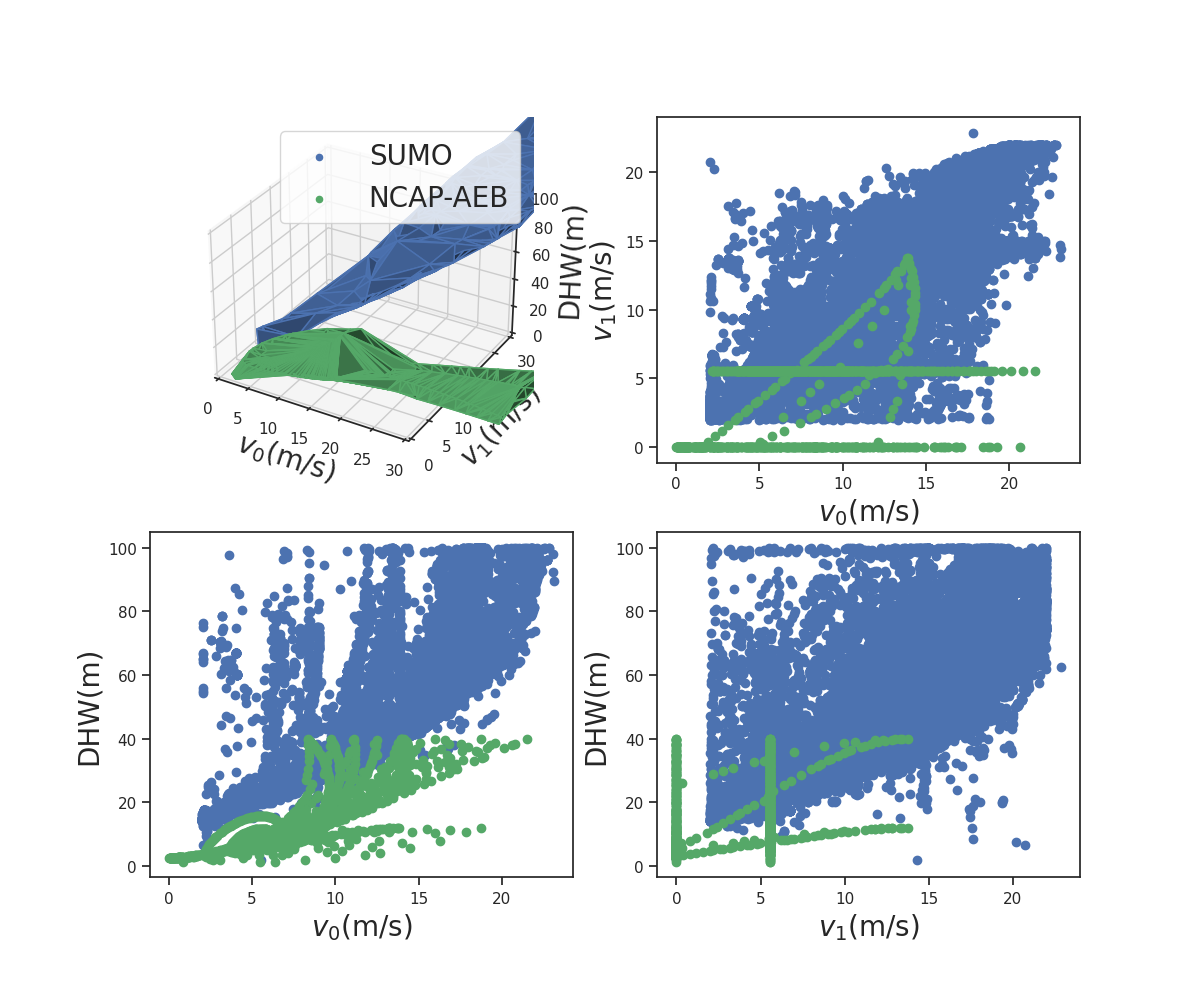}
  \caption{\small{The $\bar{\epsilon}\alpha$-almost safe sets obtained for IDM\_0 in the NCAP-AEB case.}}
  \label{fig:idm_0}
\end{subfigure}%
\hspace{1em}%
\begin{subfigure}{0.32\linewidth}
  \centering
  \includegraphics[trim={2cm 1cm 3cm 1cm},clip,width=0.95\textwidth]{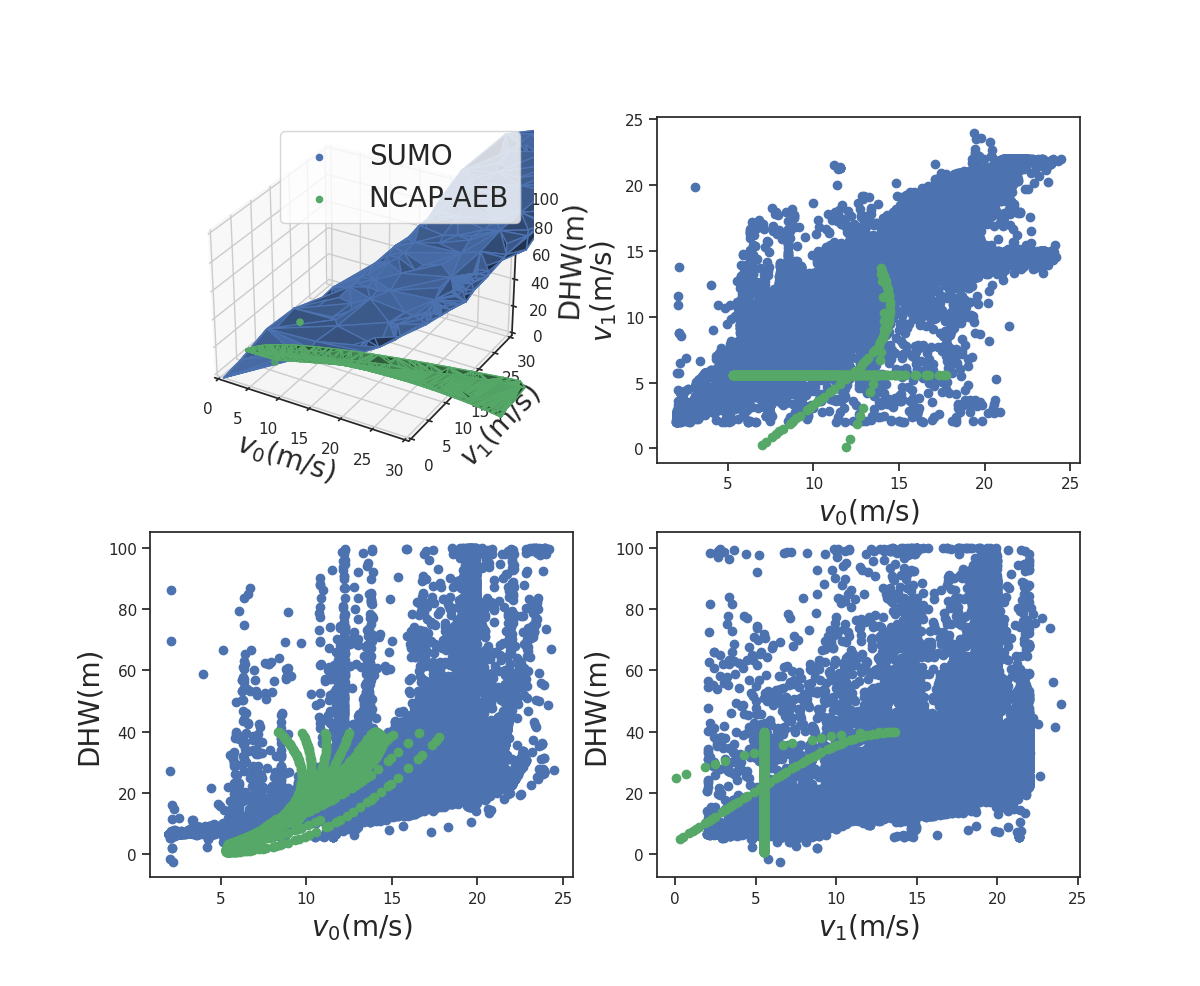}
  \caption{\small{The $\bar{\epsilon}\alpha$-almost safe sets obtained for IDM\_1 in the NCAP-AEB case.}}
  \label{fig:idm_1}
\end{subfigure}%
\hspace{1em}%
\begin{subfigure}{0.32\linewidth}
  \centering
  \includegraphics[trim={2cm 1cm 2cm 1cm},clip,width=0.95\textwidth]{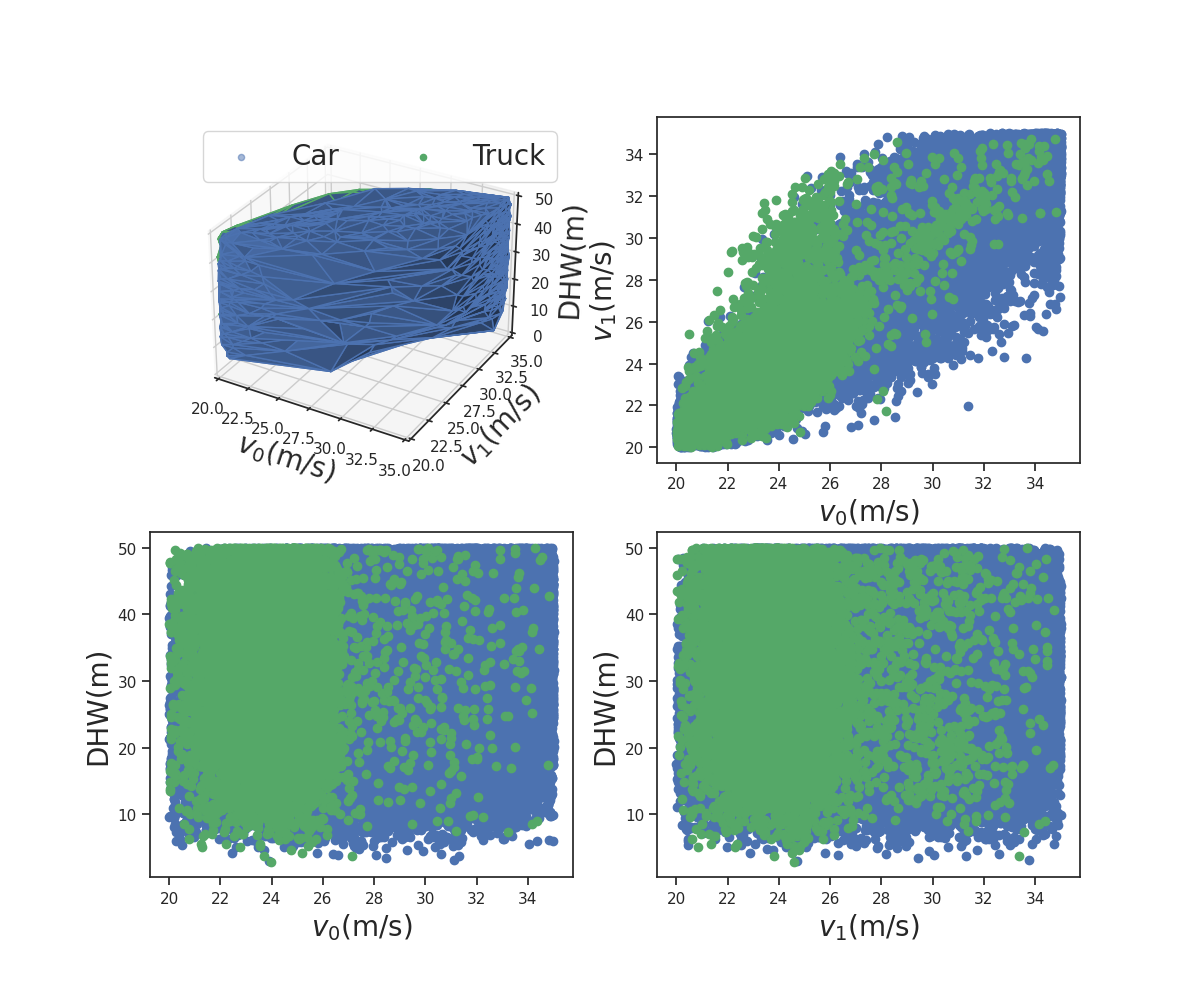}
  \caption{\small{Comparing the $\bar{\epsilon}\alpha$-almost safe sets obtained from the class of car drivers and the class of truck drivers in the HighD case.}}
  \label{fig:highD_leadfollow}
\end{subfigure}
\caption{\small{Comparing the $\bar{\epsilon}\alpha$-almost safe sets obtained from various cases for the lead-vehicle following domain.}}
\label{fig:lead_follow}
\vspace{-3mm}
\end{figure*}

Finally, note that the selection of $\alpha$ for the set construction follows the procedure described in Section~\ref{subsec:construction}, with the initial lower and upper bounds of $\alpha$ set to $0.01$ and $100$, respectively. These bounds are set for computational time concerns, and were found heuristically to have a reasonable computation time while ensuring that we obtained the same results as with looser bounds. For all the cases presented as examples in the section, changing the bounds to $[0, \inf)$ will not change the final result but will increase the time required to explore the optimal $\alpha$. The search terminates if the best $\alpha$-shape that wraps $\mathcal{D}_s$ in a single polytope is found (the termination threshold is 0.1). That is, to a certain extent, the proposed algorithm not only finds the almost safe operable domain, but also finds the \emph{optimal} almost safe operable domain. Recall that the $\alpha$-shape construction algorithm runs at a time complexity of $O( N \log{N} )$, for the high-dimensional ODD analysis with a significantly large data set, such as the multi-vehicle interaction domain with the HighD data set, we also implement a hierarchical $k$-means clustering routine to divide the data points into several clusters until all clusters are smaller than a chosen preset threshold. The final $\alpha$-shape is then determined by combining all $\alpha$-shapes derived from the obtained clusters. In practice, the computationally convenient threshold value is dependent upon the OSS dimension and the computational capability available for the analysis. On a computer with a 2.3 GHz CPU, We use 100000 and 1000 points for the three-dimensional lead-vehicle following domain and the 13-dimensional multi-vehicle interaction domain, respectively. Throughout this section, we also have $\beta=0.001$. As a result, all obtained $\bar{\epsilon}$ values are derived with a confidence level of at least 0.999.

\subsection{The lead-vehicle following domain}
We start with the lead-vehicle following domain and analyze three different cases including HighD data set, SUMO-based simulation, and a customized simulation executing the NCAP-AEB car-to-car testing procedure. For the HighD data set, we define $\mathcal{S}$ with $p_{\max}=50 \text{ m}, v_{\min}=20 \text{ m/s}$, and $v_{\max}=35 \text{ m/s}$ and consider all vehicles in all lanes. The extracted trajectories are further classified into two categories determined by the SV's type (car or truck). For the SUMO simulation, we consider $\mathcal{S}$ with $p_{\max}=100 \text{ m}, v_{\min}=0 \text{ m/s}$, and $v_{\max}=30 \text{ m/s}$. For the NCAP-AEB case, the testing procedure defines the $\mathcal{S}$ as $p_{\max}=40 \text{ m}, v_{\min}=0 \text{ m/s}$, and $v_{\max}=25 \text{ m/s}$. The experiment results are summarized in Table~\ref{tab:lead_follow} and illustrated in Fig.~\ref{fig:lead_follow}. 

In Table~\ref{tab:lead_follow}, the safe travel distance and the inferred fatality rate are not available (N/A) for some of the cases as there are collisions included in $\mathcal{D}$. The TTC is presented with the average value and the standard deviation of all admissible states. All TTC values are clipped at 9 seconds for convenience as the large TTC is of very little value for safety analysis. A TTC is valid if it is positive (i.e. $v_0>v_1$). The TTC validate rate is determined as the ratio between the total number of time steps with valid TTC and the total number of time steps in the data. Within each case and each column, the bold font emphasizes the value that indicates the higher-risk driving behavior (e.g., small average TTC and small $\alpha$-shape occupancy). We acknowledge that the safety indication through TTC is not necessarily based on the adopted criterion only, but may also admit other variants such as a certain threshold (e.g., the time-to-collision violation in~\cite{wishart2020driving}), yet those variants fundamentally share the same problem of model predictive collisions as discussed in Section~\ref{sec:introduction} and \cite{weng2021model}. We further emphasize some observations as follows.

First, the comparison between SUMO and NCAP-AEB presents an interesting case of the same set of SVs tested with different testing policies induced by the traffic vehicle behavior. From the statistical summary shown in Table~\ref{tab:lead_follow}, with the TTC based evaluation, IDM\_1 is considered more dangerous in both the NCAP-AEB and the SUMO cases, yet the valid rates are different. On the other hand, the proposed operational domain specific and unbiased safety evaluation considers IDM\_1 as the mostly safer behavior because it exhibits a higher probability of remaining inside the operable domain dictated by the $\alpha$-shape (small $\bar{\epsilon}$) with a higher density. However, note that both the illustrated $\alpha$-shapes in Fig.~\ref{fig:idm_0}, Fig.~\ref{fig:idm_1}, and the occupancy values in Table~\ref{tab:lead_follow} illustrate that the obtained safe operable domains from the two cases (SUMO and NCAP-AEB) are different. Recalling Remark~\ref{rmk:idms}, IDM\_1 has a lower tolerance of short following distances. This deficiency is more pronounced in the NCAP-AEB case with a more aggressive lead-vehicle driving behavior than the SUMO case. IDM\_1 thus ends up with a relatively smaller safe operable domain than IDM\_0 in the NCAP-AEB case, whereas in the SUMO case, the IDM\_1's safe operable domain is larger. In summary, IDM\_1's low tolerance of short following distances leads to a safer behavior than IDM\_0 in the normal driving environment, yet it also confines itself to a smaller safe operable domain in the NCAP-AEB case which is more biased towards the falsification purpose with the other traffic behaving relatively aggressively.

Second, for the HighD case, the proposed metric mostly agrees with the mileage-based fatality rate measure and identifies the naturalistic behavior induced by the class of truck drivers as more dangerous than that of the class of car drivers, even though the class of car drivers exhibits a smaller average TTC. This contradicts the commonly held interpretation of lower TTC-based metric values as an indication of less safe driving behaviors. It also aligns with the well-known deficiency of TTC also reported by other work in the literature~\cite{weng2020model, weng2021model}.

Finally, throughout all the cases, given the same confidence level, the $\bar{\epsilon}$ values from the proposed metric all exhibit a significantly smaller magnitude than the fatality rate value (mostly ten-thousand times smaller). Fundamentally, this occurs because the inferred fatality rate from~\cite{fraade2018measuring} does not have an explicitly defined operable domain. As a result, one would require a larger data set to obtain a similar level of probability to that of our proposed domain-specific metric.

\subsection{Multi-vehicle Interaction and Vehicle-to-Pedestrian Interaction}
\begin{figure}
\vspace{2mm}
\centering
\begin{subfigure}{\linewidth}
  \centering
  \includegraphics[trim={1.cm 0 0.5cm 0cm},clip,width=0.99\textwidth]{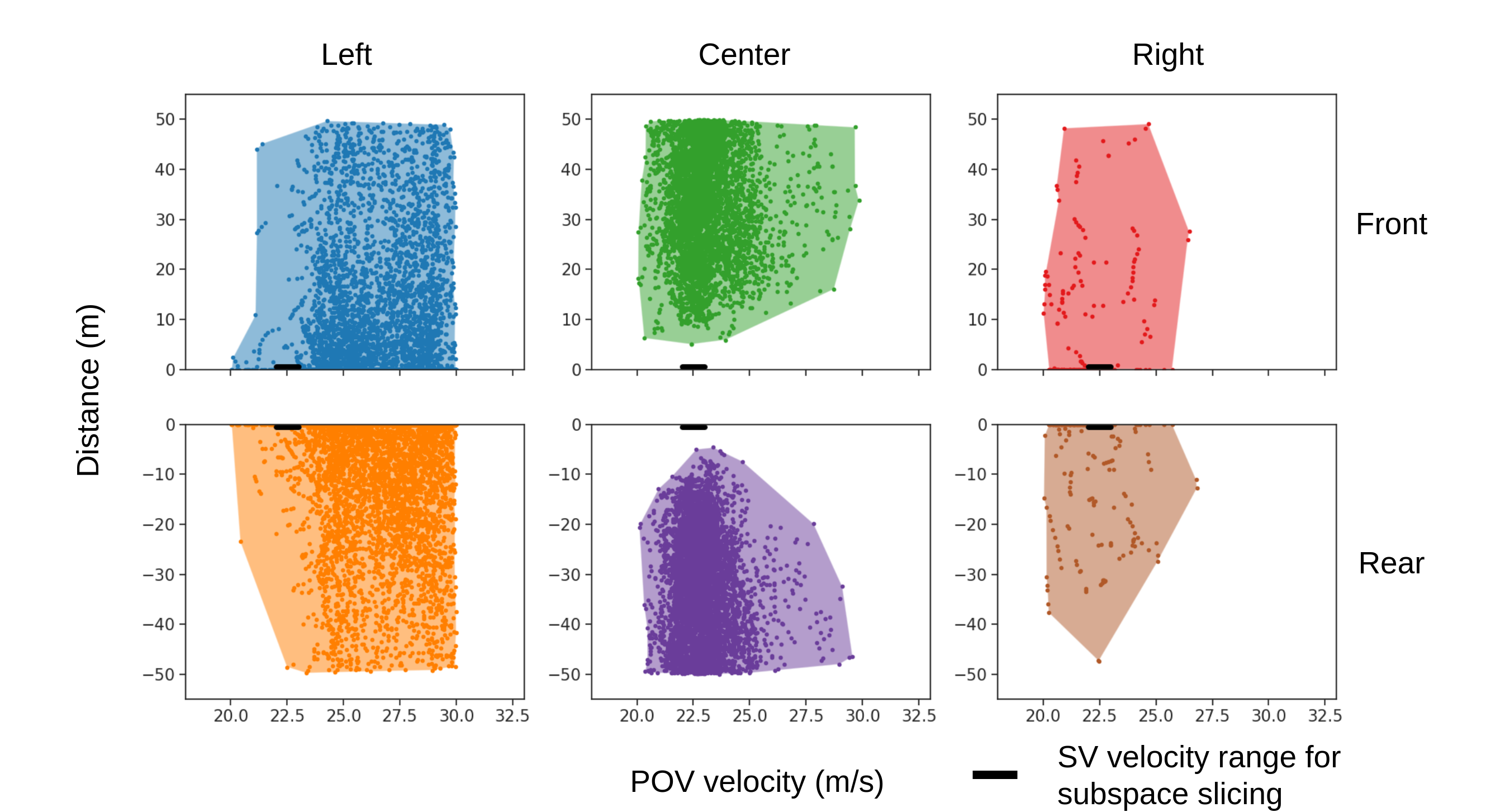}
  \caption{\small{Subspace slicing for $v_0 \in [22,23]$ with truck being the SV.}}
  \label{fig:highD2D_truck_lowspeed}
\end{subfigure}
\begin{subfigure}{\linewidth}
  \centering
  \includegraphics[trim={1.cm 0 0.5cm 0cm},clip,width=0.99\textwidth]{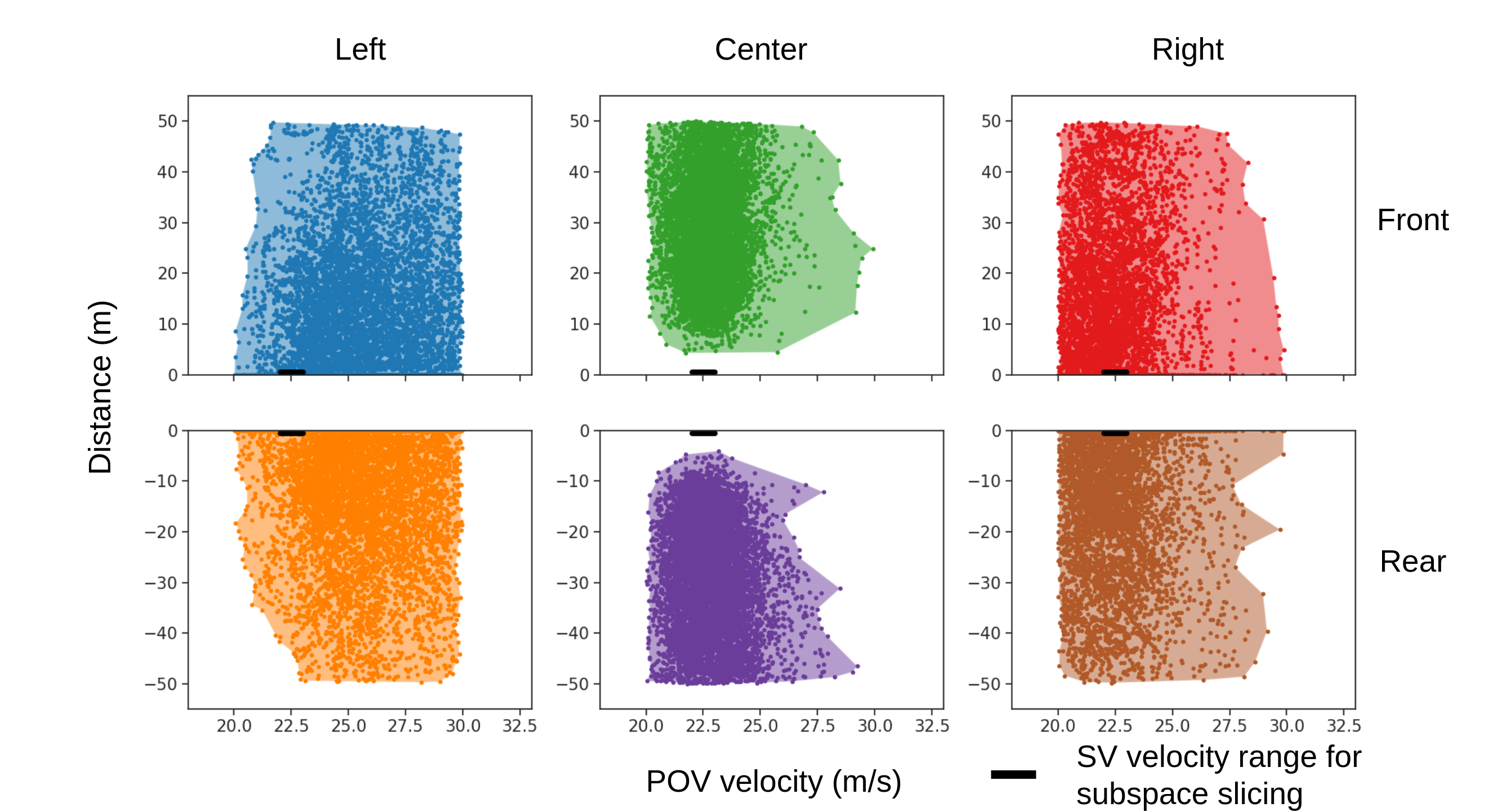}
  \caption{\small{Subspace slicing for $v_0 \in [22,23]$ with car being the SV.}}
  \label{fig:highD2D_car_lowspeed}
\end{subfigure}
\begin{subfigure}{\linewidth}
  \centering
  \includegraphics[trim={1.cm 0 0.5cm 0cm},clip,width=0.99\textwidth]{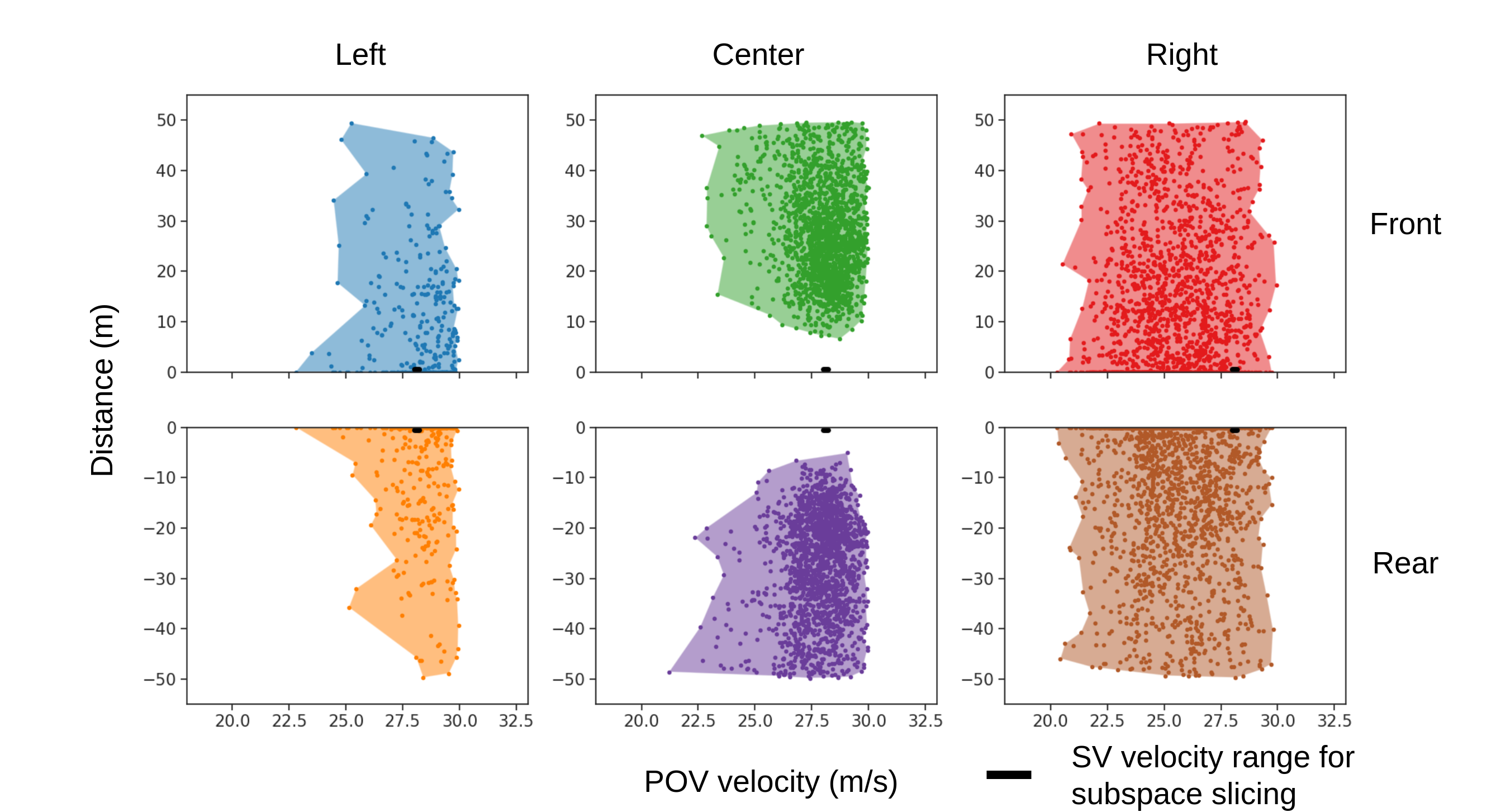}
  \caption{\small{Subspace slicing for $v_0 \in [28,28.5]$ with car being the SV.}}
  \label{fig:highD2D_car_highspeed}
\end{subfigure}
\caption{\small{Comparing the subspace slicing of the $\bar{\epsilon}\alpha$-almost safe sets obtained from HighD data set for the multi-vehicle interaction domain with different SV velocity ranges.}}
\label{fig:highD_closedall}
\vspace{-3mm}
\end{figure}

\begin{table}
    \vspace{2mm}
    \centering
    \resizebox{0.49\textwidth}{!}{%
    \begin{tabular}{c|c|c||c|c|c}
    \hline
    SV     & Safe Distance (km) & 1-R(C=0.999) & $\bar{\epsilon} (\times 10^{-4})$ (C=0.999) & $\frac{|\mathcal{D}_s|}{|\Phi_{\alpha}(\mathcal{D_s})|}(\times 10^{-12})$  & $\frac{|\Phi_{\alpha}(\mathcal{D})|}{|\mathcal{S}|}$ \\ \hline
    Car    & 536.895                       & 0.0205                & 0.0004              & {\bf 0.1507}             & 0.3505       \\ \cline{1-1}
    Truck  & \textbf{168.042}              & \textbf{0.0640}       & \textbf{0.0012}     & 0.3461      & {\bf 0.0441}                \\ \hline
    \end{tabular}%
    }
    \caption{Safety study of the multi-vehicle interaction domain with the HighD data set.}
    \label{tab:highD_allclosed}
    \vspace{-5mm}
\end{table}

This section starts with a case study of the HighD case with the multi-vehicle interaction domain. As the driving environment in the HighD case consists of only straight-road segments, the domain specification shown in Fig.~\ref{fig:model} directly applies with $p_{\max}=50 \text{ m}, p_{\min}=-50 \text{ m}, v_{\min}=20 \text{ m/s}, v_{\max}=30 \text{ m/s}$. The domain extraction for the left and right regions is confined to the adjacent lanes near the SV's lane and also excludes some SV lanes with light traffic on the side (mostly with lane ID 5). The results are summarized in Table~\ref{tab:highD_allclosed} and Fig.~\ref{fig:highD_closedall}. 

For the statistical inferred fatality rate, the truck is considered more dangerous with a short safe travel distance. On the other hand, for the proposed method, the truck is considered more dangerous with a large $\bar{\epsilon}$ and a large occupancy value, yet the point density is also large. Moreover, comparing the center column subplots between Fig.~\ref{fig:highD2D_truck_lowspeed} and Fig.~\ref{fig:highD2D_car_lowspeed}, at least within the inspected SV velocity range, the car and the truck share similar lead-following distances (indicated by the bottom of the front-center subplots in green) and rear-following distances (indicated by the bottom of the rear-center subplots in purple). That is, the other traffic vehicles are not staying further away from a truck than they do from a car, nor do trucks maintain a longer following distance from the lead vehicle than cars. This contradicts the intuitive opinion one typically forms about naturalistic driving behavior in the real-world. Finally, comparing the first column with the third column in all three subplots in Fig.~\ref{fig:highD_closedall}, one observes that vehicles on the left typically travel at a faster speed than the SV. This aligns with the nature of the data set as the HighD data set is primarily collected from German highways. This observation may not hold as we examine other driving environments, as we shall soon demonstrate. 

As we consider the Waymo case and the Carla case, where pedestrian information is available, the vehicle-pedestrian interaction domain is involved. Although both cases involve a variety of driving environments including highway, urban roads, intersections, and roundabout, the domain definition considers them as unknown disturbances and uncertainties. The specifications shown in Fig.~\ref{fig:model} still apply. For both cases, we have $p_{\max}=50 \text{ m}, p_{\min}=-50 \text{ m}, v_{\min}=1 \text{ m/s}, v_{\max}=25 \text{ m/s}$. The side subregion falls between the lateral offsets of $2.5 \text{ m}$ and $10 \text{ m}$ from the SV's geometric center. The results are summarized in Table~\ref{tab:v2vp}. Fig.~\ref{fig:other_closedall} illustrates the multi-vehicle interaction domain sub-space analysis where the vehicle-pedestrian states (the last 4 dimensions in the 17-dimensional OSS) are ignored. The vehicle-to-pedestrian interaction is analyzed separately in Fig.~\ref{fig:ped}.

\begin{figure}
\centering
\vspace{2mm}
\begin{subfigure}{\linewidth}
  \centering
  \includegraphics[trim={1.cm 0 0.5cm 0cm},clip,width=0.99\textwidth]{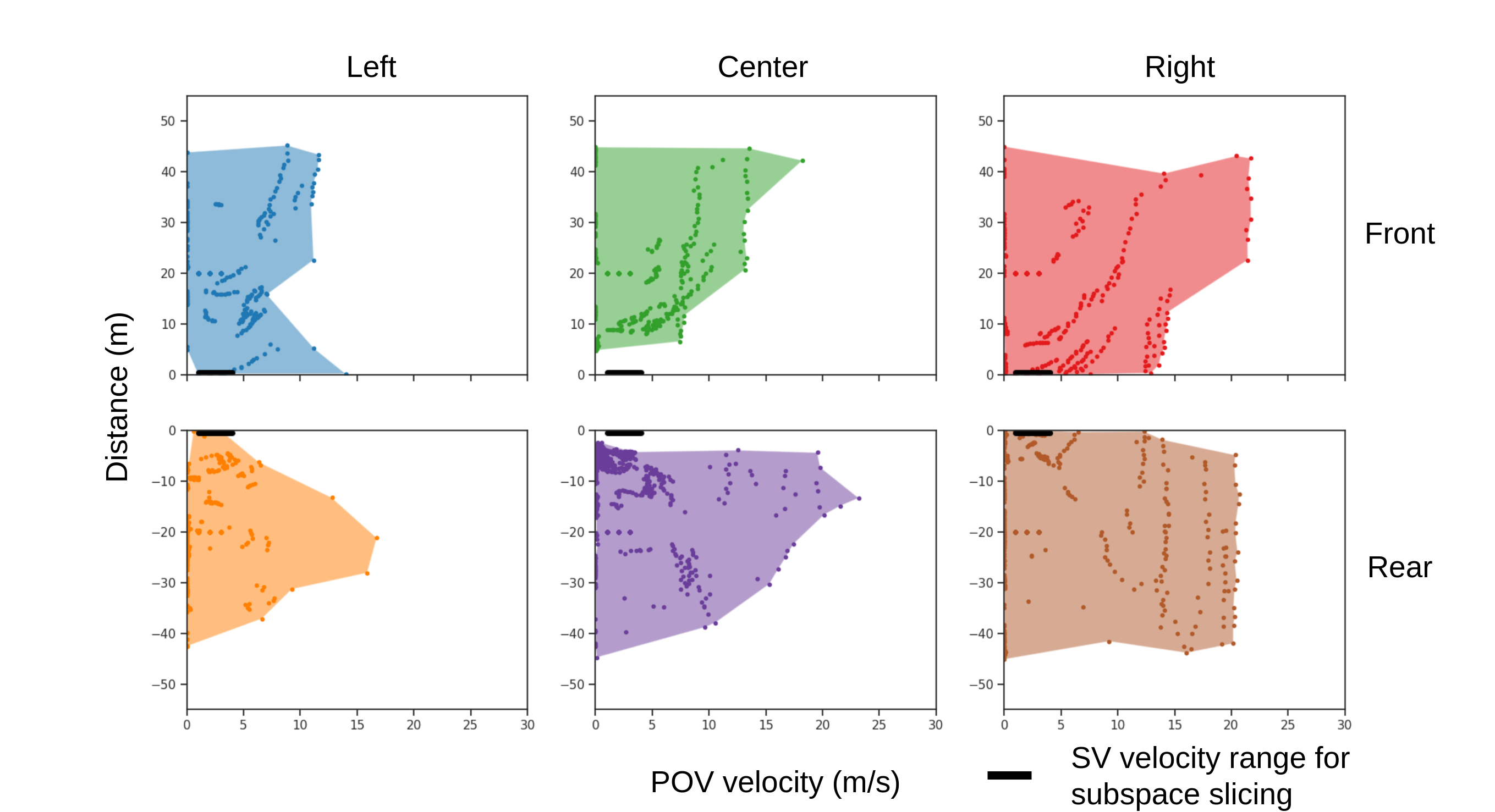}
  \caption{\small{Subspace slicing for $v_0 \in [1,4] \text{ m/s}$ with Waymo-SDC being the SV.}}
  \label{fig:waymo_closedall_lowspeed}
\end{subfigure}
\begin{subfigure}{\linewidth}
  \centering
  \includegraphics[trim={1.cm 0 0.5cm 0cm},clip,width=0.99\textwidth]{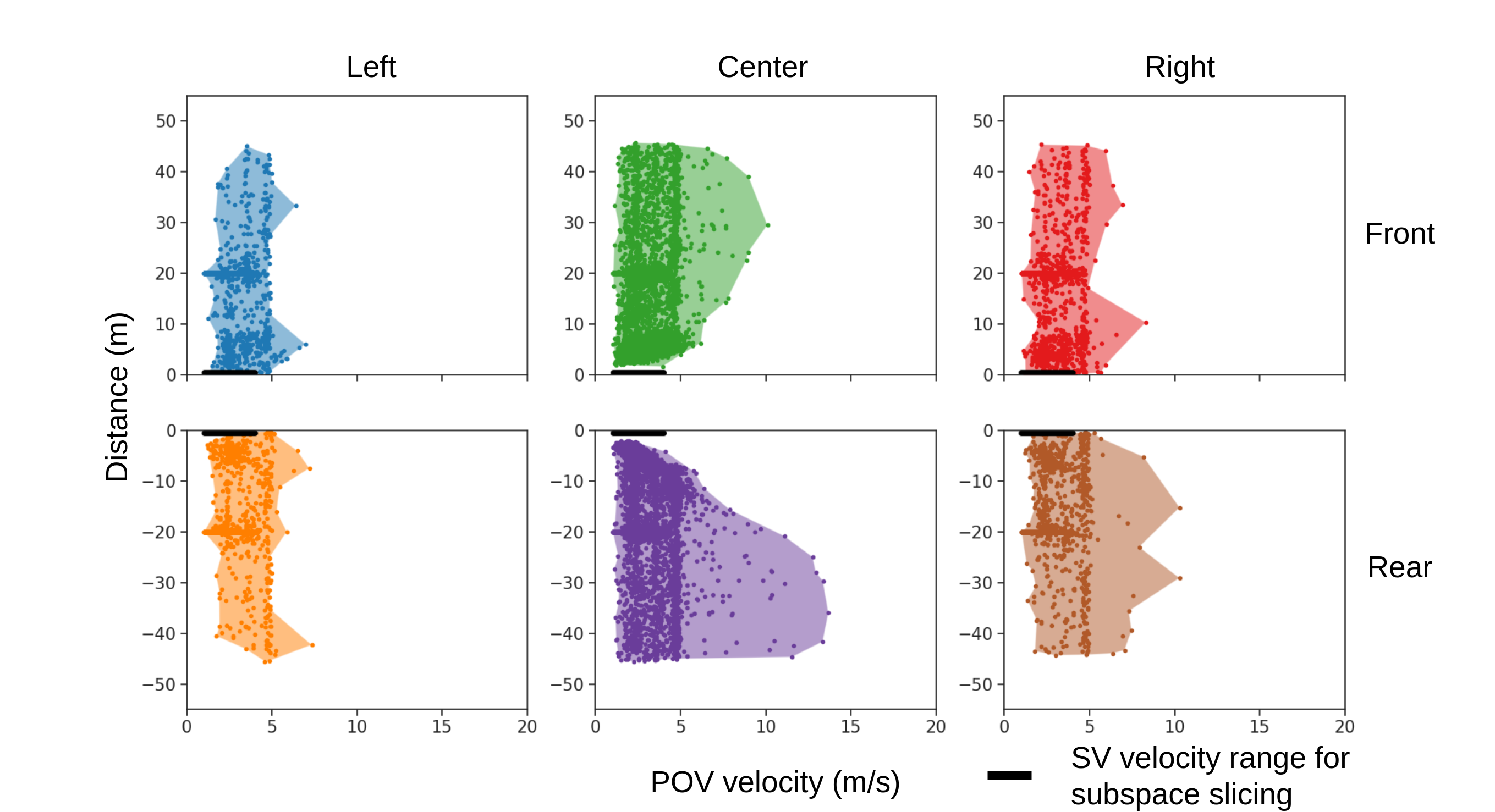}
  \caption{\small{Subspace slicing for $v_0 \in [1,4] \text{ m/s}$ with Carla-Autopilot being the SV.}}
  \label{fig:carla_closedall_lowspeed}
\end{subfigure}
\begin{subfigure}{\linewidth}
  \centering
  \includegraphics[trim={1.cm 0 0.5cm 0cm},clip,width=0.99\textwidth]{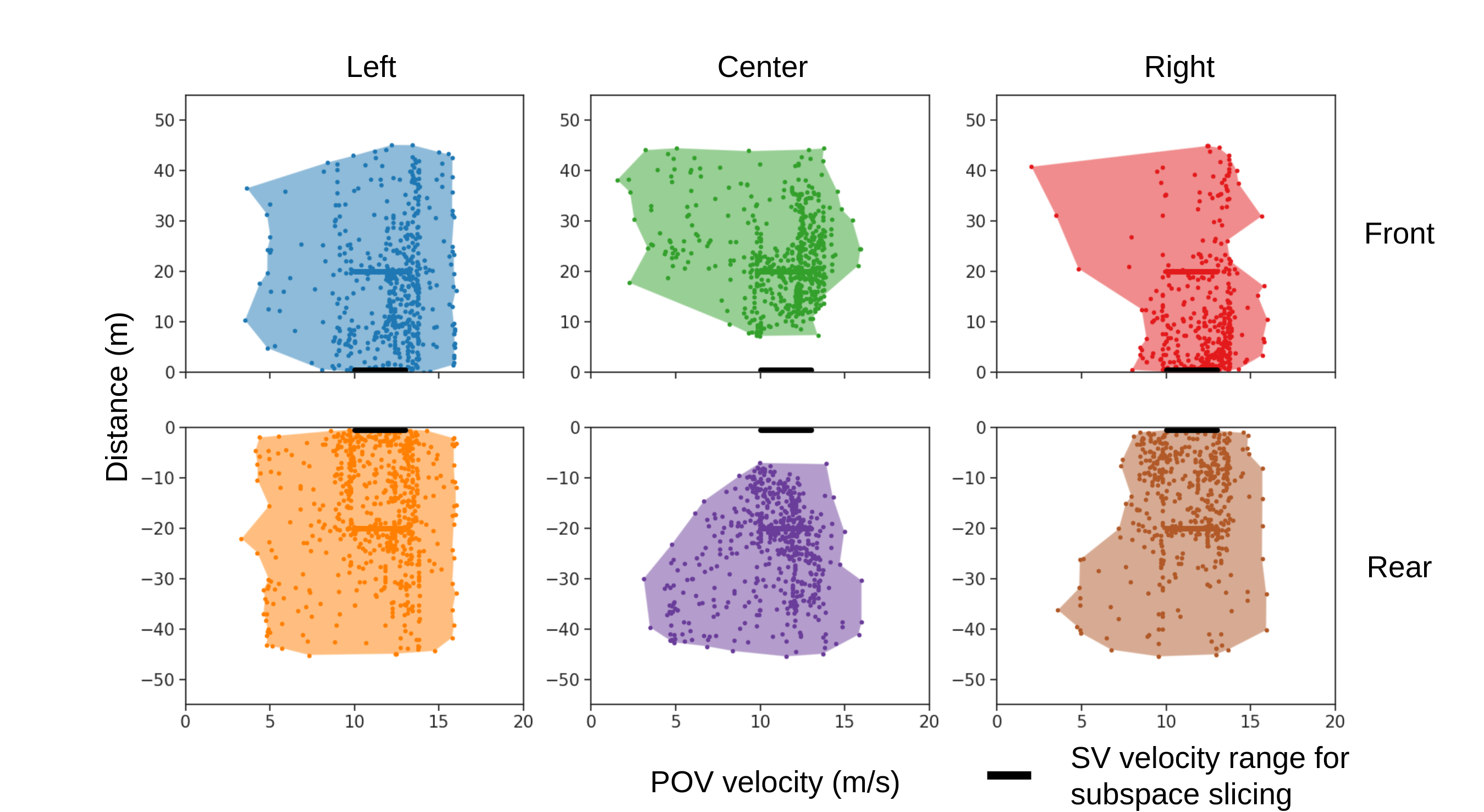}
  \caption{\small{Subspace slicing for $v_0 \in [10,14] \text{ m/s}$ with Carla-Autopilot being the SV.}}
  \label{fig:carla_closedall_highspeed}
\end{subfigure}
\caption{\small{Comparing the subspace slicing of the $\bar{\epsilon}\alpha$-almost safe sets obtained from the Waymo and Carla cases for the multi-vehicle interaction domain with different SV velocity ranges and ignoring the vehicle-pedestrian states.}}
\label{fig:other_closedall}
\vspace{-5mm}
\end{figure}

\begin{table}[b]
    \centering
    \resizebox{0.49\textwidth}{!}{%
    \begin{tabular}{c|c|c||c|c|c}
    \hline
    SV     & Safe Distance (km) & 1-R(C=0.999) & $\bar{\epsilon} (\times 10^{-4})$ (C=0.999) & $\frac{|\mathcal{D}_s|}{|\Phi_{\alpha}(\mathcal{D_s})|}(\times 10^{-24})$  & $\frac{|\Phi_{\alpha}(\mathcal{D})|}{|\mathcal{S}|}$ \\ \hline
    Waymo  & 40.778                            & 0.2386      & 8.8567     & 1.4591        & 0.0096        \\ \hline
    Carla  & 399.195                           & 0.0275      & 0.8060     & 18.3606       & 0.0118        \\ \hline
    \end{tabular}%
    }
    \caption{Safety study of the combined multi-vehicle interaction domain and the vehicle-pedestrian domain in the Waymo and the Carla cases.}
    \label{tab:v2vp}
\end{table}

\begin{figure}[]
\centering
\vspace{2mm}
\begin{subfigure}{\linewidth}
  \centering
  \includegraphics[trim={2cm 0cm 2cm 1cm},clip,width=0.9\textwidth]{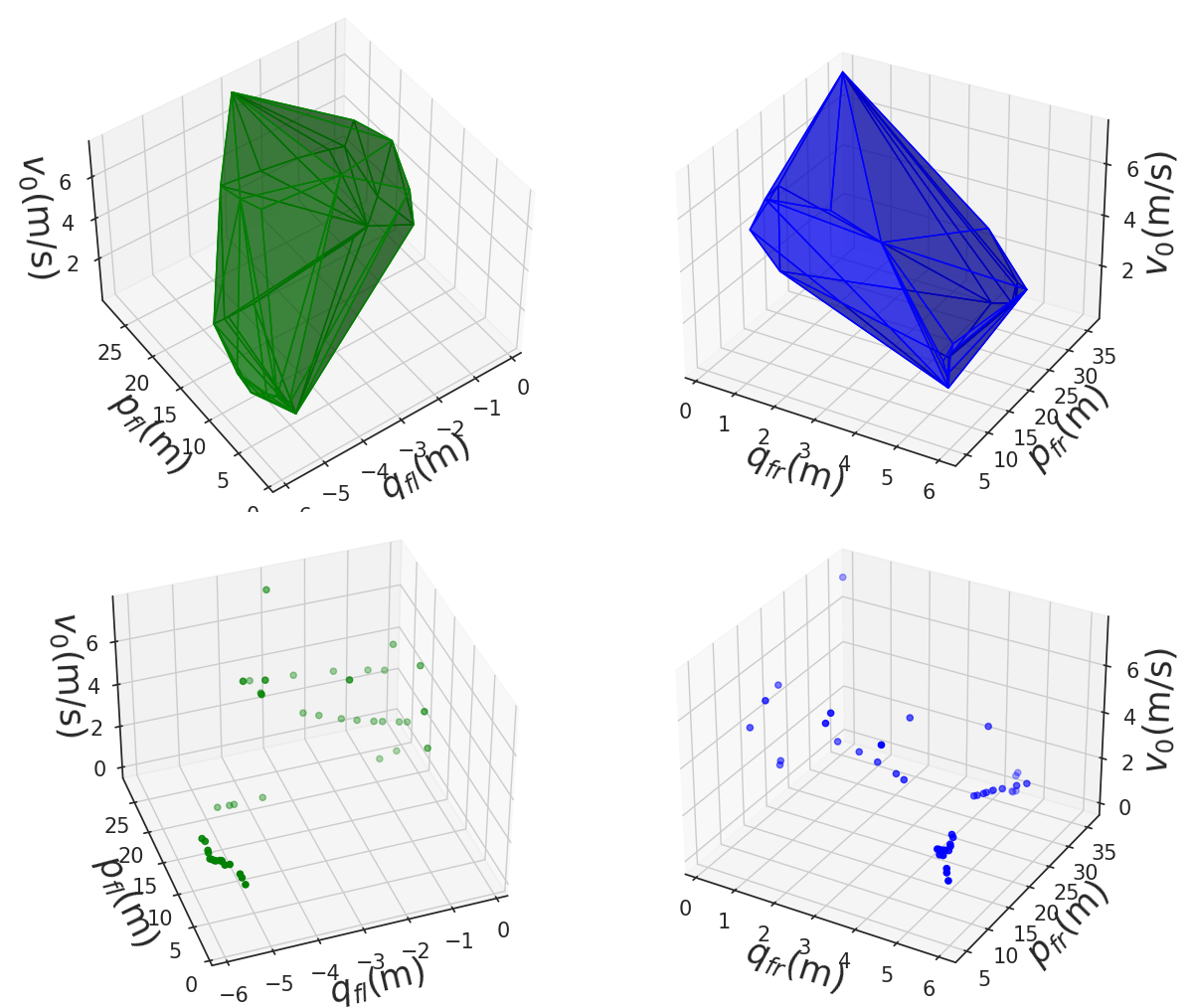}
  \caption{\small{Waymo case.}}
  \label{fig:waymo_ped}
\end{subfigure}
\begin{subfigure}{\linewidth}
  \centering
  \includegraphics[trim={2cm 0cm 2cm 1cm},clip,width=0.9\textwidth]{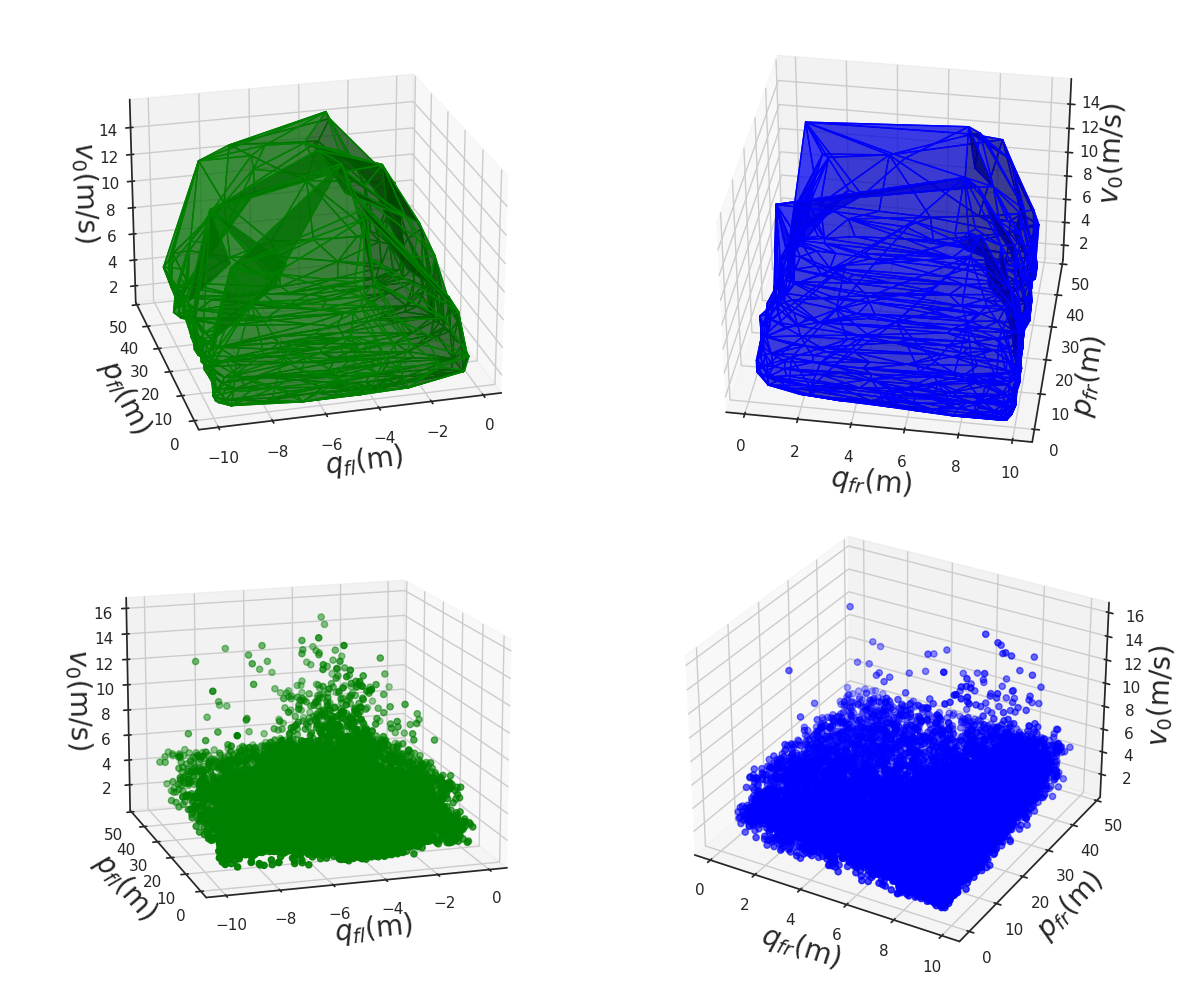}
  \caption{\small{Carla case.}}
  \label{fig:carla_ped}
\end{subfigure}
\caption{\small{Subspace slicing of the vehicle-pedestrian interaction domain: the first row on both sub-figures denotes the obtained $\alpha$-shapes and the second row denotes the points in $\mathcal{D}_s$ within the selected sub-space.}}
\label{fig:ped}
\vspace{-7mm}
\end{figure}


Within the SV velocity range of $[1,4]$(m/s) (see the center column of Fig.~\ref{fig:waymo_closedall_lowspeed} and Fig.~\ref{fig:carla_closedall_lowspeed}), the Waymo-SDC is more conservative as it maintains a longer following distance. Moreover, the observation also generalizes to the vehicle-pedestrian interaction domain where the Carla-Autopilot exhibits a short vehicle-pedestrian distance within a large velocity range (see Fig.~\ref{fig:ped}). However, note that these comparisons are not necessarily fair as the driving environments are essentially different. 

In comparison with the HighD case, the two analyzed cases have poorer coverage performance, and vehicles are mostly operating at a low speed range given that the driving environments are different. The observation from the HighD case where vehicles on the left run faster is no longer valid as illustrated in Fig.~\ref{fig:other_closedall}. The advantage of having a domain-specific safety analysis can also be shown through the Waymo case in Table~\ref{tab:v2vp}. Limited by the data availability, the total safe travel distance for the Waymo-SDC is short, leading to a large fatality rate (0.2386). In addition, the $\bar{\epsilon}$ value is much smaller for the same confidence level.

\section{Conclusion and Discussions} \label{sec:conclusion}
This paper has presented a novel safety metric that is operational domain specific and provably unbiased for performance evaluation of CAVs involving the $\alpha$-shape and the $\epsilon$-almost robustly forward set invariance property. The performance of the proposed method is also demonstrated over several commonly encountered and challenging ODDs with a variety of data sets collected with different fidelity levels. It is shown, provably and empirically, more accurate than many leading measures, observed and predictive safety lagging measures. In comparison with the inferred fatality rate, the domain-specific nature also customizes a more precise safety assessment property. 

As discussed in Section~\ref{sec:main}, it is of future interest to consider application examples that illustrate the $\bar{\epsilon}\alpha$-almost safe set metric with expanded and richer information related to the dynamic modeling, engagement information, and other safety related features. It is also of practical value to explore more efficient algorithms in deriving the (optimal) $\alpha$-shape.

\bibliographystyle{IEEEtran}
\bibliography{mybibfile}

\begin{thebibliography}{10}
\providecommand{\url}[1]{#1}
\csname url@samestyle\endcsname
\providecommand{\newblock}{\relax}
\providecommand{\bibinfo}[2]{#2}
\providecommand{\BIBentrySTDinterwordspacing}{\spaceskip=0pt\relax}
\providecommand{\BIBentryALTinterwordstretchfactor}{4}
\providecommand{\BIBentryALTinterwordspacing}{\spaceskip=\fontdimen2\font plus
\BIBentryALTinterwordstretchfactor\fontdimen3\font minus
  \fontdimen4\font\relax}
\providecommand{\BIBforeignlanguage}[2]{{%
\expandafter\ifx\csname l@#1\endcsname\relax
\typeout{** WARNING: IEEEtran.bst: No hyphenation pattern has been}%
\typeout{** loaded for the language `#1'. Using the pattern for}%
\typeout{** the default language instead.}%
\else
\language=\csname l@#1\endcsname
\fi
#2}}
\providecommand{\BIBdecl}{\relax}
\BIBdecl

\bibitem{feng2020testing}
S.~Feng, Y.~Feng, C.~Yu, Y.~Zhang, and H.~X. Liu, ``Testing scenario library
  generation for connected and automated vehicles, part i: Methodology,''
  \emph{IEEE Transactions on Intelligent Transportation Systems}, 2020.

\bibitem{altekar2021infrastructure}
N.~Altekar, S.~Como, D.~Lu, J.~Wishart, D.~Bruyere, F.~Saleem, and K.~L. Head,
  ``Infrastructure-based sensor data capture systems for measurement of
  operational safety assessment (osa) metrics,'' \emph{SAE Technical Papers},
  no. 2021, 2021.

\bibitem{fraade2018measuring}
L.~Fraade-Blanar, M.~S. Blumenthal, J.~M. Anderson, and N.~Kalra,
  \emph{{Measuring automated vehicle safety: Forging a framework}}, 2018.

\bibitem{censi2019liability}
A.~Censi, K.~Slutsky, T.~Wongpiromsarn, D.~Yershov, S.~Pendleton, J.~Fu, and
  E.~Frazzoli, ``Liability, ethics, and culture-aware behavior specification
  using rulebooks,'' in \emph{2019 International Conference on Robotics and
  Automation (ICRA)}.\hskip 1em plus 0.5em minus 0.4em\relax IEEE, 2019, pp.
  8536--8542.

\bibitem{favaro2018autonomous}
F.~Favar{\`o}, S.~Eurich, and N.~Nader, ``{Autonomous vehicles’
  disengagements: Trends, triggers, and regulatory limitations},''
  \emph{Accident Analysis \& Prevention}, vol. 110, pp. 136--148, 2018.

\bibitem{schwall2020waymo}
M.~Schwall, T.~Daniel, T.~Victor, F.~Favaro, and H.~Hohnhold, ``Waymo public
  road safety performance data,'' \emph{arXiv preprint arXiv:2011.00038}, 2020.

\bibitem{bowen2020presentation}
B.~Weng, ``Modeled exploration of proposed safety assessment metrics for
  {ADS},'' 2020, {SAE} Government Industry Meeting.

\bibitem{weng2021model}
------, ``A class of model predictive safety performance metrics for driving
  behavior evaluation,'' in \emph{2021 IEEE International Intelligent
  Transportation Systems Conference (ITSC)}, 2021, pp. 180--187.

\bibitem{wang2021review}
C.~Wang, Y.~Xie, H.~Huang, and P.~Liu, ``A review of surrogate safety measures
  and their applications in connected and automated vehicles safety modeling,''
  \emph{Accident Analysis \& Prevention}, vol. 157, p. 106157, 2021.

\bibitem{lee1976theory}
D.~N. Lee, ``A theory of visual control of braking based on information about
  time-to-collision,'' \emph{Perception}, vol.~5, no.~4, pp. 437--459, 1976.

\bibitem{wishart2020driving}
J.~Wishart, S.~Como, M.~Elli, B.~Russo, J.~Weast, N.~Altekar, E.~James, and
  Y.~Chen, ``Driving safety performance assessment metrics for ads-equipped
  vehicles,'' \emph{SAE Technical Paper}, vol.~2, no. 2020-01-1206, 2020.

\bibitem{shalev2017formal}
S.~Shalev-Shwartz, S.~Shammah, and A.~Shashua, ``On a formal model of safe and
  scalable self-driving cars,'' \emph{arXiv preprint arXiv:1708.06374}, 2017.

\bibitem{every2017novel}
J.~L. Every, F.~Barickman, J.~Martin, S.~Rao, S.~Schnelle, and B.~Weng, ``A
  novel method to evaluate the safety of highly automated vehicles,'' in
  \emph{25th International Technical Conference on the Enhanced Safety of
  Vehicles (ESV) National Highway Traffic Safety Administration, Detroit,
  Michigan}, 2017.

\bibitem{junietz2018criticality}
P.~Junietz, F.~Bonakdar, B.~Klamann, and H.~Winner, ``Criticality metric for
  the safety validation of automated driving using model predictive trajectory
  optimization,'' in \emph{2018 21st International Conference on Intelligent
  Transportation Systems (ITSC)}.\hskip 1em plus 0.5em minus 0.4em\relax IEEE,
  2018, pp. 60--65.

\bibitem{weng2020model}
B.~Weng, S.~J. Rao, E.~Deosthale, S.~Schnelle, and F.~Barickman, ``Model
  predictive instantaneous safety metric for evaluation of automated driving
  systems,'' in \emph{2020 IEEE Intelligent Vehicles Symposium (IV)}.\hskip 1em
  plus 0.5em minus 0.4em\relax IEEE, 2020, pp. 1899--1906.

\bibitem{ding2011toward}
J.~Ding, J.~H. Gillula, H.~Huang, M.~P. Vitus, W.~Zhang, and C.~J. Tomlin,
  ``Toward reachability-based controller design for hybrid systems in
  robotics.''

\bibitem{aasljung2019probabilistic}
D.~{\AA}sljung, M.~Westlund, and J.~Fredriksson, ``A probabilistic framework
  for collision probability estimation and an analysis of the discretization
  precision,'' in \emph{2019 IEEE Intelligent Vehicles Symposium (IV)}.\hskip
  1em plus 0.5em minus 0.4em\relax IEEE, 2019, pp. 52--57.

\bibitem{collin2021plane}
A.~Collin, A.~Bin-Nun, and R.~Duintjer, ``Plane and sample: maximizing
  information about autonomous vehicle performance using submodular
  optimization,'' in \emph{2021 IEEE Intelligent Transportation Systems
  Conference (ITSC)}.\hskip 1em plus 0.5em minus 0.4em\relax IEEE, 2021.

\bibitem{hejase2020methodology}
M.~Hejase, U.~Ozguner, M.~Barbier, and J.~Ibanez-Guzman, ``A methodology for
  model-based validation of autonomous vehicle systems,'' in \emph{2020 IEEE
  Intelligent Vehicles Symposium (IV)}.\hskip 1em plus 0.5em minus 0.4em\relax
  IEEE, 2020, pp. 2097--2103.

\bibitem{akkiraju1995alpha}
N.~Akkiraju, H.~Edelsbrunner, M.~Facello, P.~Fu, E.~Mucke, and C.~Varela,
  ``Alpha shapes: definition and software,'' in \emph{Proceedings of the 1st
  international computational geometry software workshop}, vol.~63, 1995,
  p.~66.

\bibitem{weng2021towards}
B.~Weng, L.~J. Capito~Ruiz, U.~Ozguner, and K.~Redmill, ``Towards guaranteed
  safety assurance of automated driving systems with scenario sampling: An
  invariant set perspective,'' \emph{IEEE Transactions on Intelligent
  Vehicles}, 2021.

\bibitem{weng2021formal}
B.~Weng, L.~Capito, U.~Ozguner, and K.~Redmill, ``A formal characterization of
  black-box system safety performance with scenario sampling,'' \emph{IEEE
  Robotics and Automation Letters}, 2021.

\bibitem{arief2021deep}
M.~Arief, Z.~Huang, G.~K.~S. Kumar, Y.~Bai, S.~He, W.~Ding, H.~Lam, and
  D.~Zhao, ``Deep probabilistic accelerated evaluation: A robust certifiable
  rare-event simulation methodology for black-box safety-critical systems,'' in
  \emph{International Conference on Artificial Intelligence and
  Statistics}.\hskip 1em plus 0.5em minus 0.4em\relax PMLR, 2021, pp. 595--603.

\bibitem{fan2017d}
C.~Fan, B.~Qi, S.~Mitra, and M.~Viswanathan, ``D ry vr: data-driven
  verification and compositional reasoning for automotive systems,'' in
  \emph{International Conference on Computer Aided Verification}.\hskip 1em
  plus 0.5em minus 0.4em\relax Springer, 2017, pp. 441--461.

\bibitem{hauer2020clustering}
F.~Hauer, I.~Gerostathopoulos, T.~Schmidt, and A.~Pretschner, ``Clustering
  traffic scenarios using mental models as little as possible,'' in \emph{2020
  IEEE Intelligent Vehicles Symposium (IV)}.\hskip 1em plus 0.5em minus
  0.4em\relax IEEE, 2020, pp. 1007--1012.

\bibitem{yan2021distributionally}
X.~Yan, S.~Feng, H.~Sun, and H.~X. Liu, ``Distributionally consistent
  simulation of naturalistic driving environment for autonomous vehicle
  testing,'' \emph{arXiv preprint arXiv:2101.02828}, 2021.

\bibitem{alphashape2011}
K.~Fischer, ``Introduction to alpha shapes,'' 2011, unpublished.

\bibitem{kengithub}
B.~Ken, G.~Neil, and K.~Phillip, ``alphashape,''
  \url{https://github.com/bellockk/alphashape}, 2019.

\bibitem{selby1965index}
B.~Selby, ``The index of dispersion as a test statistic,'' \emph{Biometrika},
  vol.~52, no. 3/4, pp. 627--629, 1965.

\bibitem{dang2008sensitive}
T.~Dang, A.~Donz{\'e}, O.~Maler, and N.~Shalev, ``Sensitive state-space
  exploration,'' in \emph{2008 47th IEEE Conference on Decision and
  Control}.\hskip 1em plus 0.5em minus 0.4em\relax IEEE, 2008, pp. 4049--4054.

\bibitem{capito2021bpa}
L.~Capito, K.~Redmill, and U.~Ozguner, ``{Model-based decomposition and
  backtracking framework for probabilistic risk assessment in automated vehicle
  systems},'' in \emph{2021 International Topical Meeting on Probabilistic
  Safety Assessment and Analysis, PSA 2021}.\hskip 1em plus 0.5em minus
  0.4em\relax American Nuclear Society, 2021.

\bibitem{krajewski2018highd}
R.~Krajewski, J.~Bock, L.~Kloeker, and L.~Eckstein, ``The highd dataset: A
  drone dataset of naturalistic vehicle trajectories on german highways for
  validation of highly automated driving systems,'' in \emph{2018 21st
  International Conference on Intelligent Transportation Systems (ITSC)}.\hskip
  1em plus 0.5em minus 0.4em\relax IEEE, 2018, pp. 2118--2125.

\bibitem{sun2020scalability}
P.~Sun, H.~Kretzschmar, X.~Dotiwalla, A.~Chouard, V.~Patnaik, P.~Tsui, J.~Guo,
  Y.~Zhou, Y.~Chai, B.~Caine \emph{et~al.}, ``Scalability in perception for
  autonomous driving: Waymo open dataset,'' in \emph{Proceedings of the
  IEEE/CVF Conference on Computer Vision and Pattern Recognition}, 2020, pp.
  2446--2454.

\bibitem{SUMO2018}
\BIBentryALTinterwordspacing
P.~A. Lopez, M.~Behrisch, L.~Bieker-Walz, J.~Erdmann, Y.-P. Fl{\"o}tter{\"o}d,
  R.~Hilbrich, L.~L{\"u}cken, J.~Rummel, P.~Wagner, and E.~Wie{\ss}ner,
  ``Microscopic traffic simulation using sumo,'' in \emph{The 21st IEEE
  International Conference on Intelligent Transportation Systems}.\hskip 1em
  plus 0.5em minus 0.4em\relax IEEE, 2018. [Online]. Available:
  \url{https://elib.dlr.de/124092/}
\BIBentrySTDinterwordspacing

\bibitem{treiber2013traffic}
M.~Treiber and A.~Kesting, ``Traffic flow dynamics,'' \emph{Traffic Flow
  Dynamics: Data, Models and Simulation, Springer-Verlag Berlin Heidelberg},
  2013.

\bibitem{dosovitskiy2017carla}
A.~Dosovitskiy, G.~Ros, F.~Codevilla, A.~Lopez, and V.~Koltun, ``Carla: An open
  urban driving simulator,'' in \emph{Conference on robot learning}.\hskip 1em
  plus 0.5em minus 0.4em\relax PMLR, 2017, pp. 1--16.

\bibitem{van2017euro}
M.~R. van Ratingen, ``The {EURO} {NCAP} safety rating,'' in
  \emph{Karosseriebautage Hamburg 2017}.\hskip 1em plus 0.5em minus 0.4em\relax
  Springer, 2017, pp. 11--20.

\end{thebibliography}
\end{document}